\documentclass[runningheads]{llncs}

\usepackage{eccv}

\usepackage{eccvabbrv}

\usepackage{graphicx}
\usepackage{booktabs}
\usepackage{fancyvrb}

\usepackage[accsupp]{axessibility}  

\usepackage[pagebackref,breaklinks,colorlinks,citecolor=eccvblue]{hyperref}
\usepackage{hyperref}
\usepackage{amsmath}

\usepackage{orcidlink}

\newcommand{\model}{Argos\xspace}
\usepackage{mdframed}
\newmdenv[
  backgroundcolor=gray!20,
  linecolor=gray!60,
  roundcorner=6pt,
  innerleftmargin=10pt,
  innerrightmargin=10pt,
  innertopmargin=8pt,
  innerbottommargin=8pt
]{grayblock}

\usepackage[most]{tcolorbox}
\usepackage{fvextra}

\usepackage{booktabs}
\newcommand{\spatialresultrow}[5]{#1 & #2 & #3 & #4 & #5 \\}

\usepackage{booktabs}
\newcommand{\ablateresultrow}[9]{#1 & #2 & #3 & #4 & #5 & #6 & #7 & #8 & #9\\}

\newcommand{\resultrow}[4]{#1 & #2 & #3 & #4\\}

\usepackage[most]{tcolorbox}
\newtcolorbox{AIBox}[1]{%
  enhanced,
  colback=white,            
  colframe=black,           
  boxrule=0.8pt,
  arc=2mm,                  
  outer arc=2mm,
  left=4mm,right=4mm,top=1mm,bottom=1mm,
  before skip=12pt,after skip=12pt,
  fonttitle=\bfseries\color{white},
  title={#1},
  attach boxed title to top center={yshift=-2mm},
  boxed title style={       
    colback=black,
    colframe=black,
    boxrule=0pt,
    arc=2mm,
    outer arc=2mm,
    left=6mm,right=6mm,top=1mm,bottom=1mm,
  },
}

\begin{document}

\title{Multimodal Reinforcement Learning with Adaptive Verifier for AI Agents} 

\titlerunning{Argos}

\author{
Reuben Tan$^{1}$ \quad
Baolin Peng$^{1}$ \quad
Zhengyuan Yang$^{1}$ \quad
Hao Cheng$^{1}$ \quad
Oier Mees$^{1}$ \\
Theodore Zhao$^{1}$ \quad
Andrea Tupini$^{1}$ \quad
Isar Meijer$^{1}$ \quad
Qianhui Wu$^{1}$ \\
Yuncong Yang$^{2}$ \quad
Lars Liden$^{1}$ \quad
Yu Gu$^{1}$ \quad
Sheng Zhang$^{1}$ \quad
Xiaodong Liu$^{1}$ \quad
Lijuan Wang$^{1}$ \quad
Marc Pollefeys$^{1,3}$ \quad
Yong Jae Lee$^{4}$ \quad
Jianfeng Gao$^{1}$ \\
$^{1}$Microsoft Research \quad
$^{2}$UMass Amherst \quad
$^{3}$ETH Zurich \quad
$^{4}$UW--Madison
}

\authorrunning{F.~Author et al.}


\maketitle

\begin{center}
\includegraphics[width=\linewidth]{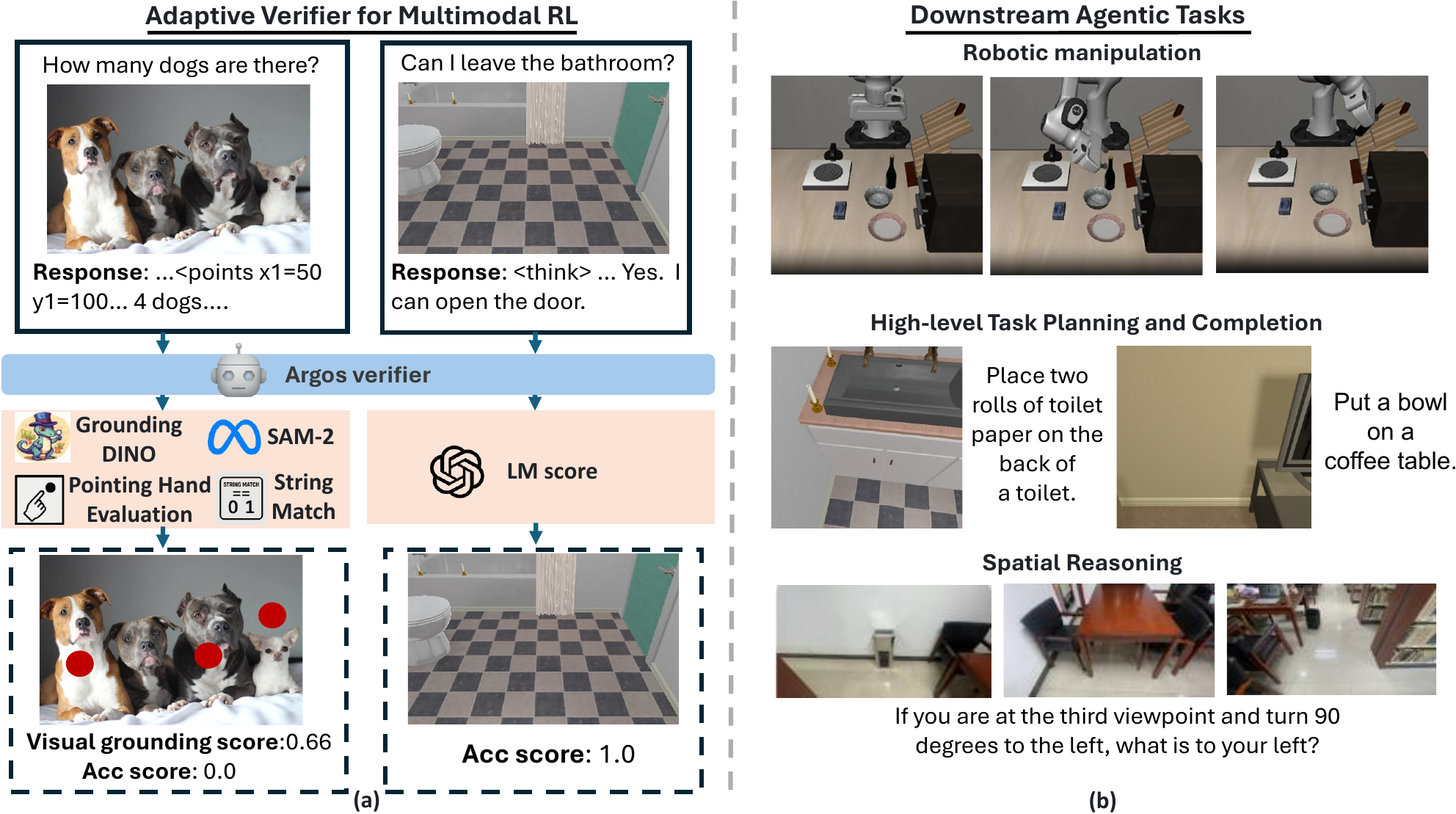}
        \vspace{-0.2in}
   \captionof{figure}{\textbf{Multimodal RL with our \model adaptive verifier.} We propose to train agentic foundation models using an adaptive verifier \model that dynamically selects different scoring tools based on the training sample during the RL stage. Then, we evaluate the resulting model on multiple agentic benchmarks including embodied task planning and completion as well as spatial reasoning.
   } 
\label{fig:motivation_fig}
\end{center}

\begin{abstract}
  Agentic reasoning models trained with multimodal reinforcement learning (MMRL) have become increasingly capable, yet they are almost universally optimized using sparse, outcome-based rewards computed based on the final answers. Richer rewards computed from the reasoning tokens can improve learning significantly by providing more fine-grained guidance. However, it is challenging to compute more informative rewards in MMRL beyond those based on outcomes since different samples may require different scoring functions and teacher models may provide noisy reward signals too. In this paper, we introduce the \model(\underline{A}daptive \underline{R}eward for \underline{G}rounded \& \underline{O}bjective \underline{S}coring), a principled verification framework with adaptive tool routing to train multimodal  reasoning models for agentic tasks. For each sample, \model selects from a pool of teacher-model derived and rule-based scoring functions to simultaneously evaluate: (i) final response accuracy, (ii) spatiotemporal localization of referred entities and actions, and (iii) the quality of the reasoning process. We find that by leveraging our adaptive verifier across both SFT data curation and RL training, our model achieves state-of-the-art results across multiple agentic tasks such as spatial reasoning, visual hallucination as well as robotics and embodied AI benchmarks. Critically, we demonstrate that just relying on SFT post-training on highly curated reasoning data is insufficient, as agents invariably collapse to ungrounded solutions during RL without our online verification. More importantly, we also show that adaptive composition of different scoring functions during the verification process can help to reduce reward-hacking in MMRL. 
\end{abstract}    
\vspace{-15pt}
\section{Introduction}
\label{sec:intro}
\vspace{-3pt}

Intelligent beings seamlessly integrate perception, language, and action. With a goal in mind, they first observe a scene, interpret it in context, and then formulate and execute a plan. To emulate this ability, researchers have been shifting from static perception to agentic multimodal models that can reason about observations, plan and use tools~\cite{durante2024agent,dong2024tool,wang2024genartist,li2025perception,kumar2024mmctagent}. Such agentic models for \emph{multimodal reasoning}  have wide-ranging applications, including AI agents that collaborate with humans, interactive GUI~\cite{luo2025gui} and tool-using~\cite{tian2025ego} assistants, and systems such as robots and self-driving cars. In particular, RL, including the recent GRPO~\cite{shao2024deepseekmath} and DAPO~\cite{yu2025dapo} algorithms, has been crucial in driving this progress. Verifiable outcome rewards help align such reasoning models with downstream tasks. However, using only outcome rewards provides limited guidance on the quality of the reasoning process and can cause hallucination~\cite{he2025breaking}. Motivated by these limitations, a natural direction is \emph{visually grounded reasoning}, which encourages the model to explicitly reference visual evidence such as point coordinates while it reasons. Existing approaches~\cite{sarch2025grounded,ni2025point} have explored this idea primarily through the curation of grounded SFT annotations. Crucially, we observe empirically that grounded SFT alone is insufficient: during multimodal RL (MMRL), models can still collapse to fluent but ungrounded responses when the grounding signals are not explicitly verified, especially under sparse outcome-based rewards. This motivates a training-time verifier that can reliably enforce consistency between perception and language and curb reward hacking.


While RL for text-only reasoning has been extensively studied, approaches for computing richer rewards in MMRL remain comparatively under-explored and introduce unique challenges, such as selecting appropriate scoring functions per sample, mitigating noisy signals from teacher models, and maintaining consistency between perception and language throughout the reasoning process. To address the above-mentioned challenges, we introduce Adaptive Reward for Grounded and Objective Scoring (\model) verifier (Figure~\ref{fig:motivation_fig}a), which dynamically selects from a set of teacher models and rule-based scoring functions like string matching to evaluate the response of each sample across spatial grounding, reasoning quality and accuracy.
Our proposed verifier jointly evaluates final answer accuracy, spatiotemporal grounding and reasoning quality. \model is a general framework and can be extended to include increasing capable task-specific models. Finally, we compute an  aggregated final reward that is gated by correct outcomes but enriched with intermediate reward terms. Additionally, we propose an approach based on overlaying explicit 2D point coordinates on images and video frames, that leverages the OCR capability of a teacher model to generate reasoning data that are visually grounded in pixels across space and time. 
In addition to MMRL, we also use \model during our data curation process to filter out low-quality rollouts from the teacher model for the SFT stage. Crucially, our \model verifier helps curb reward hacking in MMRL. \model is also related to research on tool-augmented agents~\cite{schick2023toolformer} but those methods generally employ tools for inference-time problem solving, leaving the intermediate reasoning and visual evidence under-verified during training. In contrast, \model helps to convert multiple noisy reward signals into a final verifiable reward.  


From a learning perspective, \model reframes MMRL as multi-objective optimization with multiple noisy teacher rewards. We provide a brief theoretical justification, along with a detailed analysis in the supplemental, to provide an intuition on why adaptive and multi-objective reward verification can help the policy model to learn better by guiding it towards global Pareto optimality. The modular architecture of \model enables it to extend naturally to new modalities and objectives. As task-specific teacher models improve, our \model has the potential to compute more informative reward signals, enabling the training of more capable and robust multimodal reasoning agents. In conclusion, we summarize our contributions as follows:
\begin{enumerate}
    \item We propose \model, that is used during data curation to filter out low-quality annotations and to provide aggregated and verifiable rewards during MMRL. Also, we introduce a novel data curation pipeline for generating reasoning traces that are visually grounded in space and time.

    \item We demonstrate the effectiveness of \model in achieving state-of-the-art results on multiple agentic benchmarks (Figure~\ref{fig:motivation_fig}b)  against similarly-sized models, including spatial intelligence reasoning, multimodal understanding, embodied task completion and robotics.

    \item To the best of our knowledge, we are the first work to introduce 
    a principled verification framework with adaptive tool selection for MMRL.
\end{enumerate}

\section{Related Work}
\label{sec:related}

\paragraph{Multimodal understanding and reasoning.}
The AI research community has seen rapid progress in large multimodal models (LMMs) that are able to process data from different modalities such as visual and audio information and generate reasonable responses. Before the advances in autoregressive language models, seminal vision-language models such as CLIP~\cite{radford2021learning}, ALIGN~\cite{jia2021scaling} and BLIP~\cite{li2022blip} are trained on web-scale image and language datasets, often with contrastive learning~\cite{oord2018representation}. Building upon the advances in autoregressive language models and insights on instruction tuning~\cite{chung2024scaling}, models including but not limited to Flamingo~\cite{alayrac2022flamingo}, BLIP-2/3~\cite{li2023blip,xue2025blip}, LLaVA~\cite{liu2023visual} and Mini-GPT4~\cite{zhu2023minigpt} have combined the visual capabilities of pretrained image encoders and LLMs, leveraging the latter's ability for open-ended question answering, prompting and reasoning. 
Subsequent work broadens both scope and granularity, including region-level LMMs that operate at finer spatial resolutions~\cite{guo2024regiongpt} and video-centric LMMs designed for temporal reasoning~\cite{zhang2023video,tan2024koala}. In tandem, increasingly comprehensive benchmarks have been introduced to evaluate these capabilities across tasks and modalities~\cite{liu2024mmbench,ying2024mmt}. Beyond static perception, recent works also leveraged the success of the DeepSeek-R1~\cite{guo2025deepseek} model with its proposed GRPO algorithm to train multimodal models that are capable of reasoning about images, videos and even audio~\cite{feng2025video, fan2025grit}. These works are highly relevant to earlier approaches that apply multimodal CoT such as (i) prompt-based strategies for zero/few-shot settings~\cite{zhou2022least,zhang2023multimodal}, (ii) plan-based approaches that iteratively refine intermediate thoughts and evidence~\cite{yao2023react}, and (iii) learning-based techniques that directly train models to produce rationales from paired inputs and targets~\cite{wang2022self,ye2023mplug}.

\vspace{-10pt}
\paragraph{Reinforcement Learning for reasoning and planning.}
Multimodal planning over long horizons aims to equip AI systems with the capacity to integrate and reason over streams of multimodal inputs and observations, such as language, vision, and audio among others, across extended time horizons to complete complex, goal-driven tasks in real or simulated environments~\cite{huang2023voxposer,driess2023palm,chen2022transdreamer,huang25mslmaps, mees2022calvin}. 
Recent models combine vision-language-action foundations with planning capabilities to execute open-ended tasks such as robotic control~\cite{sermanet2024robovqa,jones24fuse} and embodied navigation~\cite{yang2025embodiedbench,huang23vlmaps}, with a particular emphasis on hierarchical planning~\cite{huang2023voxposer,mees23hulc2,huang25mslmaps}. A key component here is RL, which allows agents to learn robust and generalizable policies~\cite{hafner2023mastering,hafner2025training,rosete2022corl,radosavovic2024real,nakamoto2024steering}. Others use RL to learn to use tools including but not limited to external APIs for computation~\cite{yao2023react} or predefined tools such as cropping and even using LMMs in the operation, as part of their multimodal reasoning loops~\cite{driess2023palm,myers2024policy}. Lastly, advanced models incorporate RL fine-tuning on top of pretrained LMM backbones~\cite{huang2023voxposer,hafner2023mastering,hafner2025training,driess2023palm,chen2024vision}, using environment rewards to align long-term plans with task success while retaining interpretability through intermediate subgoal generation or trajectory imagination~\cite{park2023generative,hafner2023mastering,hafner2025training,zheng2024tracevla,hatch2024videoglue}. Our work builds on the idea of tool usage by using teacher models but adaptively selects them to compute multi-objective rewards instead. 
\section{Approach}
\label{sec:approach}

\begin{figure*}[h]
\begin{center}
\includegraphics[width=0.95\linewidth]{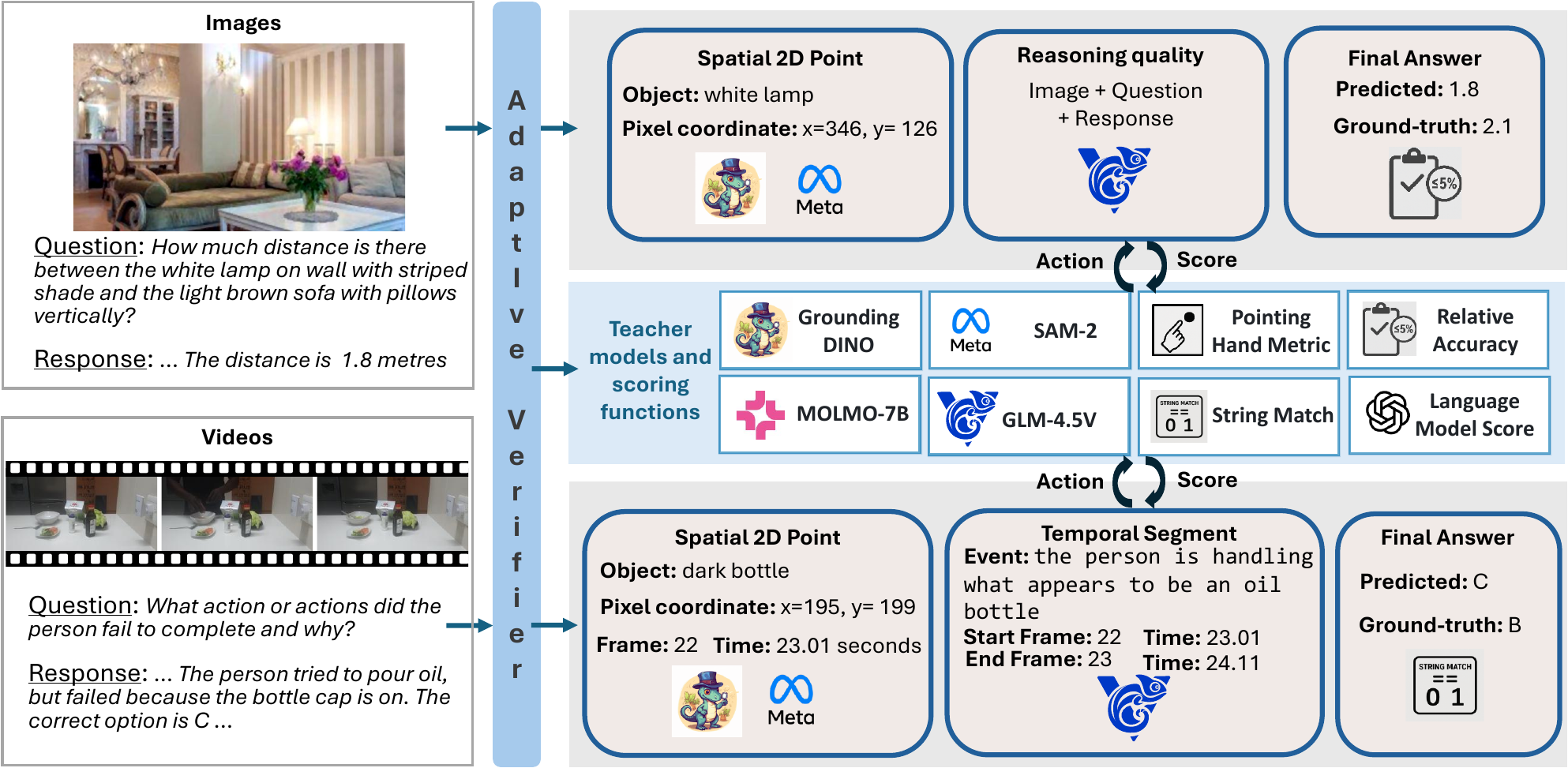}
\end{center}
\vspace{-10pt}
   \caption{\textbf{Verification process.} We use the same set of scoring functions for both images and videos. Each response is first parsed to extract information about generated 2D points, temporal segments, reasoning text and answer. Then, the adaptive verifier dynamically decides what scoring functions to call based on the extracted information. Finally, we aggregate the scores using a gated aggregation function.} 
\label{fig:agentic_verifier_scoring_process}
\vspace{-15pt}
\end{figure*} 

We present a principled and general framework for computing more informative rewards to verify the generated responses of our reasoning model by leveraging the \model verifier.\model (Figure~\ref{fig:agentic_verifier_scoring_process}) is an LMM agent that selects from a set of $K$ scoring functions to compute an aggregated reward score for each training sample. We note that the final reward computed by \model is not a process reward as it is computed at the end of the entire response. However, our reward is more informative than the conventional outcome reward since it aggregates multiple reward signals of intermediate reasoning steps such as their visual grounding accuracy. For a question and visual input, we use our multimodal reasoning model, denoted as $\pi$, to generate a response. Each response consists of a reasoning trace and final predicted answer. The adaptive verifier accepts the visual input $v$, question $q$, reasoning trace $r$, and predicted answer $\hat{y}$. Depending on the training sample, \model adaptively composes a multi-objective reward process by selecting relevant tools to score the response. Finally, it leverages a gated aggregation function to compute a final reward score. 

\subsection{Adaptive Verifier}

\textbf{How adaptive selection happens}: To begin, we parse model outputs to extract referenced spatiotemporal points and final answers from a single response string. First, the parser isolates the reasoning and answer regions: it extracts the contents of $<\text{think}></\text{think}>$ as the reasoning text, and the final answer from the $<|\text{begin\_of\_box}|><|\text{end\_of\_box}|>$ tags. In parallel, it extracts 2D point annotations from the response such as structured <points> tags with x and y coordinates, as well as temporal segments if they exist. To preserve semantic ordering, the parser records the left-to-right appearance order of all point mentions in the response. It first collects <points> spans, masks them to avoid double counting, then merges all detections by start index.



Based on these extracted fields, we prompt \model with a description of the available scoring functions and their respective limitations to allow it to select suitable functions on a per-sample basis. We provide the system prompt used in the supplemental. Then, it decides which scoring functions to call (e.g., spatiotemporal grounding tools only when points or video segments are referenced and accuracy evaluators chosen based on answer format). 

\paragraph{Spatial reward.} Our key intuition behind generating reasoning thoughts that are grounded in both space and time is that it helps to alleviate the issue of hallucination of referenced objects in reasoning traces. For images, we primarily evaluate the spatial grounding accuracy. Given the set of extracted 2D spatial points mentioned in the predicted response $P_{\text{spatial}}$, we compute the score for each point, which is also associated with a predicted object label $\hat{o}$. Our spatial reward is computed in two stages. To begin, we extract a set of $N$ generated 2D points $P = \{ (x_1, y_1, o_1), \cdot \cdot \cdot, (x_N, y_N, o_N)\}$ from a rollout generated by our model, where $x_i$, $y_i$ and $o_i$ denote the x and y coordinates, as well as the corresponding object label of the $i$-th point, respectively. For the $i$-th detected point, we use an open-vocabulary object detection model $g_\theta$ to extract a pseudo ground-truth bounding box $b^*_i$ based on the predicted object: $b^*_i = g_\theta(o_i).$
To further refine the boundaries of the object enclosed within $b^*_i$, we use a segmentation teacher model $h_\phi$ to extract a fine-grained segmentation map:
\begin{equation}
    M_i = h_\phi(b^*_i), M^*_i \in \mathbb{R}^{H \times W},
\end{equation} where $M_i$, $H$, $W$ denote the extracted segmentation mask, height and width of the input image, respectively. We compute the spatial grounding score for the $i$-th point as:
$s_i = \mathbf{1}\!\left[M_i(x_i,y_i)=1\right]$. 

In the case of images that contain synthetic visual content such as bar charts or maps, the open-vocabulary object detection model may not work well. Thus, we use another pointing model $f_\vartheta$ to generate 2D points before passing them into the segmentation teacher model $h_\phi$. Finally, we compute the spatial grounding reward $R_{\text{spatial}}$ as follows:
\vspace{-15pt}
\begin{equation}
    R_{\text{spatial}} = \frac{1}{N}\sum_{i=1}^N s_i .
\end{equation}
\vspace{-20pt}
\paragraph{Temporal rewards.} The scoring functions used to compute the spatial reward term can be easily extended to videos. When a reasoning trace references an action or event in the video that span multiple frames, it is also important to verify its existence in the video. Given a video $V$ consisting of $N_v$ frames and a question, we use an LLM to extract both frame-level observations $F$ and segment-level events or actions $E$ that span multiple frames from the generated response. We provide the query prompt used for extraction in the supplemental. 

We define $F$ as a set of $N_F$ frame-level observations $F = \{(t_1, x_1, y_1, o_1), \cdot \cdot \cdot, (t_{N_F}, x_{N_f}, y_{N_f}, o_{N_F}) \}$, where $t_i$ can either denote the relevant frame or its corresponding timestamp. In this setting, the identified frame is analogous to an image. Thus, we leverage the spatial grounding models $f_\vartheta$, $g_\theta$ and $h_\phi$ described above to compute the set of frame-level scores $S_f$.

We consider a set of segment-level events $E = \{e_i\}_{i=1}^N$, where the $i$-th event is represented as the tuple $(t_i^{\mathrm{start}},\, t_i^{\mathrm{end}},\, d_i)$ with $t_i^{\mathrm{start}}$ and $t_i^{\mathrm{end}}$ denoting the start and end times (or frame indices) of the segment, and $d_i$ the event description. Then, we query a powerful teacher reasoning model $T$ to evaluate the visual–semantic accuracy between $d_i$ and the corresponding video segment $V_{\,t_i^{\mathrm{start}}:t_i^{\mathrm{end}}}$, and return a binary score: $s_i = \text{video\_score}\!\big(d_i, V_{t_i^{\mathrm{start}}:t_i^{\mathrm{end}}}\big) \in \{0,1\}$.
 Here, $\text{video\_score}$ is the function parameterized by the reasoning teacher model. Finally, we compute the final video grounding score by computing the within-set means of the set of event scores $S_e$ and $S_f$ and defining the final score as the (unweighted) average of these means.

\vspace{-10pt}
\paragraph{Reasoning quality reward.}
Beyond evaluating the intermediate grounding of referenced objects and actions, we also evaluate the logical consistency between the generated reasoning trace and the final answer $\hat{y}$. In some cases, the policy model may generate reasonable reasoning traces but still predict the wrong answer at the end. We use a larger teacher model to compute a reasoning-quality reward as its conditional probability of the predicted response $y$ given the question $q$, reasoning trace $r$ and visual input $v$:
\begin{equation}
  R_{\text{reasoning}} \;=\; P\big(\mathrm{\hat{y}}\mid q, r, v).
\end{equation}
Intuitively, higher values indicate stronger consistency between the reasoning and the answer in the context of the question and visual input. 

\vspace{-10pt}
\paragraph{Outcome rewards.} To compute the final outcome rewards using the ground-truth answer $y^\star$, we use a combination of a language model as well as heuristic-based functions depending on the type of the question and expected answer format.

\noindent\emph{(i) Exact string match.}
For multiple-choice questions and those that require short phrases as an answer, we compute: $R_{\text{acc}} = \mathbf{1}\{\hat{y}= y^\star \}$.
\emph{(ii) Relative numerical accuracy with 5\% tolerance.}
When both answers are float numbers, we compute:
\vspace{-10pt}
\begin{equation}\label{eq:racc-float}
R_{\text{acc}}
=\mathbf{1}\!\left\{
\mathrm{relerr}(\hat{y},y^\star)\le 0.05
\right\}\enspace \text{where}
\enspace
\mathrm{relerr}(\hat{y},y^\star)=\frac{\lvert \hat{y}-y^\star\rvert}{\max(\lvert y^\star\rvert,\,1)}.
\vspace{-5pt}
\end{equation}
\emph{(iii) Binary semantic accuracy via a language model.}
We query a capable language model to assign $R_{\text{acc}}$ a value of 1 if $\hat{y}$ is semantically equivalent to $y^\star$ without any contradictions and 0 otherwise.

\paragraph{Aggregation function.} 
Finally, we aggregate the rewards from selected scoring functions using a gated scoring function to prevent potentially noisy rewards from biasing the final answer away from the correct result. We formulate the gated function as follows:
\begin{equation}\label{eq-reward}
    R_{\text{final}} = 
    \begin{cases}
        R_{\text{acc}}, & R_{\text{acc}} < \tau, \\
        \dfrac{w_A R_{\text{acc}} + w_G R_{\text{spatial}} + w_R R_{\text{reasoning}}}{w_A + w_G + w_R}, & R_{\text{acc}} \ge \tau,
    \end{cases}
\end{equation}
where $w_A$, $w_G$, $w_R$ and $\tau$ denote the weight terms for the outcome, visual grounding and reasoning quality rewards as well as the gating threshold, respectively.


While the verifier framework's modularity allows for diverse feedback, its benefit lies in the mathematical convergence when aggregating weak signals. We provide rigorous proof that composing $m$ complementary reward estimators helps suppress noise, guiding the policy towards $\delta$-Pareto-optimal solutions, meaning the model improves grounding and reasoning without sacrificing final accuracy.

\subsection{GRPO training}
Within each group of rollout reward values for the $i$-th training sample, we compute the advantage $A_i$ over the individual $j$-th reward values: $A_i = \frac{R_i - \operatorname{mean}\!\left(\{R_j\}\right)}
           {\operatorname{std}\!\left(\{R_j\}\right)}.$
Using the aggregated reward terms, we update our policy model $\pi_\theta$ using the GRPO formulation~\cite{shao2024deepseekmath}:
\vspace{-10pt}
\begin{equation}
\vspace{-10pt}
\begin{aligned}
\mathcal{J}_{\mathrm{GRPO}}(\theta)
&= \mathbb{E}_{q,\{\hat{y}_i\}}\Bigg[
\frac{1}{G}\sum_{i=1}^{G}
\min\!\bigg(
\frac{\pi_{\theta}(\hat{y}_i\mid q)}{\pi_{\theta_{\mathrm{old}}}(\hat{y}_i\mid q)}\,A_i,\;
\\[-2pt]
&\hspace{-5em}\operatorname{clip}\!\left(
\frac{\pi_{\theta}(\hat{y}_i\mid q)}{\pi_{\theta_{\mathrm{old}}}(\hat{y}_i\mid q)},\,
1-\epsilon,\,1+\epsilon
\right) A_i
\bigg)
-\beta\, D_{\mathrm{KL}}\!\left(\pi_{\theta}\,\|\,\pi_{\mathrm{ref}}\right)
\Bigg].
\end{aligned}
\end{equation}

\section{Data Curation}
\label{sec:data} 

\begin{figure*}[t!]
\begin{center}
\includegraphics[width=0.95\linewidth]{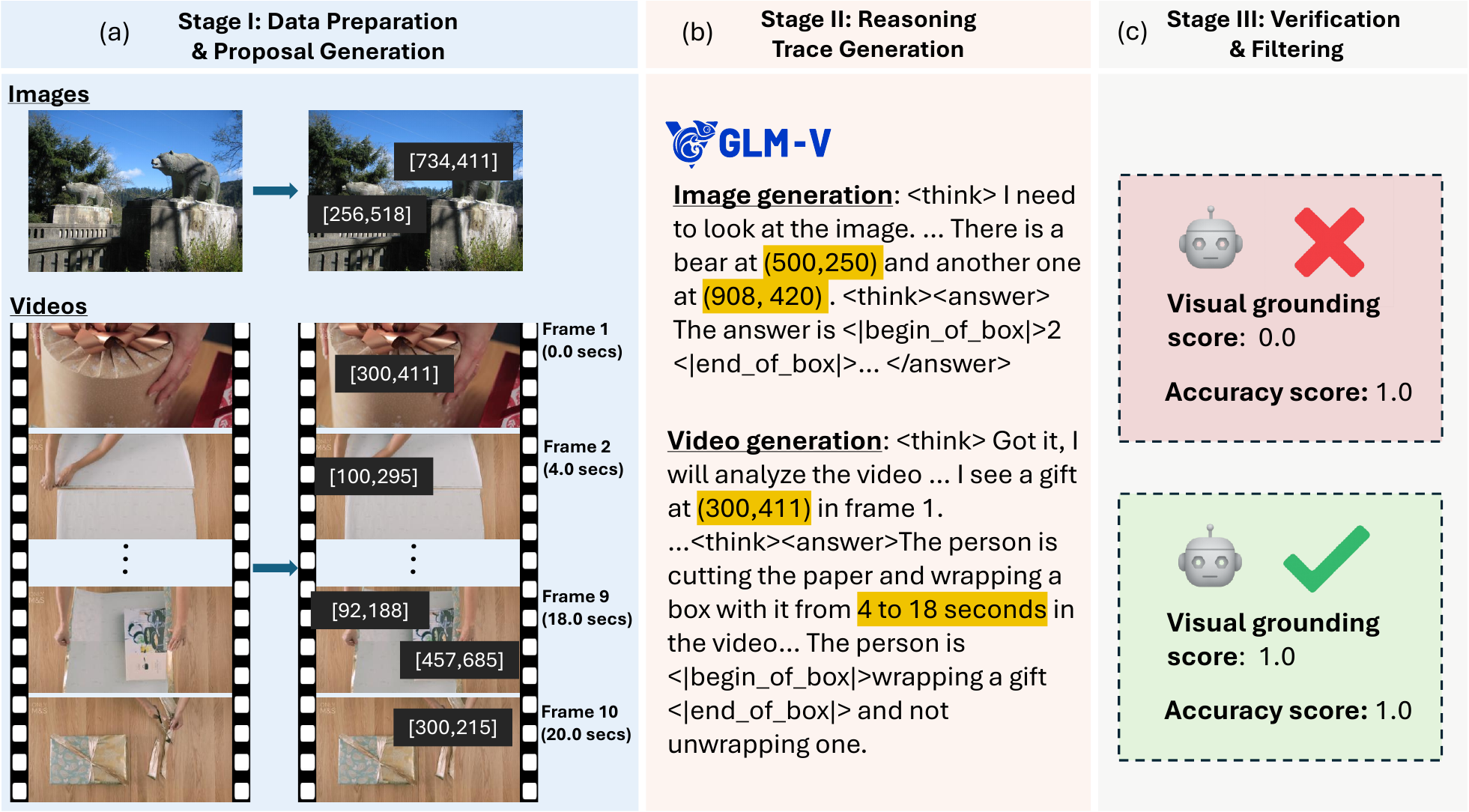}
\end{center}
\vspace{-10pt}
   \caption{\textbf{Grounded reasoning generation pipeline.} \textbf{(a) Stage I:} We extract object, action and event proposals such as 2D boxes for images and video frames as well as temporal segments for videos. 
\textbf{(b) Stage II.} We use the overlaid images and video frames to prompt a pretrained LMM to generate grounded reasoning traces that explicitly refer to these points. For videos, we also include the frame numbers and their timestamps in the query.  
\textbf{(c) Stage III.} \model adaptively scores each trace using multi-objective rewards (e.g., visual grounding and answer accuracy) and filters out samples with low-quality generations. In the image example, the sample is filtered out due to low visual grounding accuracy despite predicting the correct answer.} 
\label{fig:data_curation}
\vspace{-15pt}
\end{figure*} 

We aim to generate training data for SFT cold-start to help the base model learn to reason. We illustrate the main stages of our curation pipeline for generating reasoning traces that are visually grounded across space and time in Figure~\ref{fig:data_curation}. We primarily use the highly-capable GLM-4.1V~\cite{hong2025glm} model to generate the reasoning traces although we also use a smaller amount of generations from a proprietary model from Gemini-2.5 Flash~\cite{gemini} to augment our SFT dataset due to computational constraints.
We discuss the three main stages of our entire generation pipeline at a high-level in this section and provide more specific details on each step and examples of our curated SFT training samples in the supplemental.

\subsection{Data preparation and proposal generation} \label{sec:data_proposal}
While the GLM-4.1V teacher model can perform object localization well, it is unable to perform grounded reasoning naturally given a question and visual input. Given the visual input, question and ground-truth answer, we begin by extracting information about relevant objects, actions and events before extracting their spatial 2D and temporal positions. Based on our observation that it performs well on the task of OCR, we first use the Molmo-7B~\cite{deitke2024molmo} model to extract 2D points of relevant objects. As shown in Figure~\ref{fig:data_curation}(a), we overlay the spatial 2D points on the image or sampled video frames. Additionally, for the temporal dimension in videos, we apply a similar concept by providing explicit timestamps using both frame numbers and time in seconds. Each frame is mapped to its accurate timestamp, which is computed based on its sampling FPS. We set a maximum limit of 32 frames in experiments. In our query, we provide the input overlaid frames along with their corresponding timestamps.

\subsection{Reasoning trace generation}\label{sec:trace_gen}
After the object locations and event timestamps have been extracted, we query the GLM-4.1V model with the overlaid images and video frames (Figure~\ref{fig:data_curation}(b)) to generate reasoning traces that contain 2D points when referring to specific objects. In the case of images, we prompt the GLM-4.1V model to primarily use the visual information in the original image before referring to the coordinates on the overlaid image for reference to clarify ambiguous objects or referring expressions in its response. For videos, we prompt the teacher model to explicitly refer to 2D points in frames and multi-frame events in specific formats such as ``(x,y) in frame F (T seconds)'' as well as ``from $t_{\text{start}}$ to $t_{\text{seconds}}$''. Note that we do not include the original video frames due to computational constraints. For each sample, we generate eight possible rollouts. 

\subsection{Verification and filtering} 
Despite our curation pipeline for visually grounded reasoning traces, state-of-the-art reasoning models still produce unreliable rollouts with high frequency. For example, our yield rate was around 3.1\%. Thus, in addition to using \model to provide adaptive and dense rewards for multimodal reinforcement learning, we also use it to evaluate the generated rollouts for the training samples and filter out samples if the maximum score over all rollouts fall below a threshold value (Figure~\ref{fig:data_curation}(c)). We extract generated 2D points using regular expression and reformat them into the format: \colorbox{lightgray}{$<$point x1="x" y1="y" alt="object"$>$"object"$<$/point$>$}. Similarly, we also extract the templated timestamps in the video reasoning traces and replace them with reformatted natural language phrases. This filtering step ensures that the data used for SFT consists predominantly of visually grounded and semantically accurate reasoning examples.


\section{Experiments}
\label{sec:experiments}
In this section, we evaluate our trained model post-SFT and RL on multiple agentic benchmarks across different domains under the zero-shot setting. A key skill for multimodal AI agents to interact with their physical worlds is spatial intelligence and the ability to reason about different viewpoint perspectives and motion. We begin by evaluating on multiple vision-centric and spatial reasoning benchmarks. Next, we also examine the benefits of learning to perform grounded reasoning with \model on reducing visual hallucination since agents have to make confident and accurate predictions. Finally, we also evaluate on fine-grained robotic manipulation and high-level task planning.

\vspace{-10pt}
\paragraph{Implementation details.} We build our approach off the
publicly available Qwen2.5-VL 7B~\cite{bai2025qwen2} model and train on our curated dataset for SFT (as discussed in Section~\ref{sec:data}) and a separate and non-overlapping subset of the same dataset for the RL training with GRPO. During the MMRL stage, we set $W_A$, $W_G$,  as well as $\tau$ as 1 and $W_R$ as 0.1. We provide further training details in the supplemental.

\subsection{Spatial Reasoning Evaluations}
\begin{table}[ht]
\vspace{-20pt}
\centering
\footnotesize
\setlength{\tabcolsep}{1.5pt}
\begin{tabular}{lcccc}
\toprule
\textbf{Model} & \textbf{BLINK} & \textbf{MindCube-t} & \textbf{CV-Bench} & \textbf{CV-Bench (3D)} \\
\midrule
\spatialresultrow{Qwen2.5VL 7B~\cite{bai2025qwen2}}{54.4}{34.9}{77.0}{77.9}
\spatialresultrow{Qwen2.5VL 7B (CoT)}{53.5}{33.1}{75.6}{76.6}
\spatialresultrow{Video-R1 (SFT)}{52.7}{34.2}{75.1}{75.6}
\spatialresultrow{Video-R1 (RL)~\cite{feng2025video}}{49.0}{31.9}{60.2}{57.2}
\spatialresultrow{ViGoRL (RL)~\cite{sarch2025grounded}}{53.0}{36.4}{\textbf{80.5}}{81.0}
\spatialresultrow{\model (Ours)}{\textbf{56.0}}{\textbf{39.6}}{78.2}{\textbf{82.0}}
\bottomrule
\end{tabular}
\caption{Results on spatial reasoning benchmarks.}
\label{tab:spatial_reasoning_results}
\vspace{-25pt}
\end{table}

We report results of \model on multiple spatial reasoning benchmarks in Table~\ref{tab:spatial_reasoning_results}. For all datasets, we use accuracy (\%) as the metric.

\vspace{-10pt}
\paragraph{BLINK.} The BLINK dataset~\cite{fu2024blink} contains 14 visual perception tasks that include spatial and multiview reasoning, as well as functional correspondence. \model achieves a performance gain of over $12\%$ over the baseline Qwen2.5-VL and even outperforms Video-R1~\cite{feng2025video}, which was trained on around 160K more training samples than ours. These consistent improvements suggest that more accurate grounded reasoning may help close the gap with much larger models. 

\vspace{-10pt}
\paragraph{MindCube.} The MindCube~\cite{yin2025spatial} benchmark focuses on evaluating the ability of LMMs to perform spatial reasoning by reconstructing spatial mental models using partial observations and dynamic viewpoints. We evaluate on the tiny split which contains around 1K evaluation samples. Interestingly, CoT prompting actually hurts the performance of the base Qwen2.5VL model. In contrast, our model also outperforms the base model by over 5\%.

\vspace{-10pt}
\paragraph{CV-Bench.} CV-Bench~\cite{tong2024cambrian} is a vision-centric benchmark that assesses 2D understanding through spatial relationships and object counting, and 3D understanding through depth ordering and relative distance.
Consistent with results on other datasets, our model trained with the proposed \model gains a significant improvement over the state-of-the-art SOTA Video-R1 variants.
It is worth-noting that training with visually grounded reasoning traces in 2D images enhances the resulting model's generalization capabilities to 3D visual understanding. This appears to be corroborated by performance gains achieved by our model on embodied AI tasks, that we discuss in later sections.

\begin{table}[ht]
\vspace{-15pt}
\centering
\footnotesize
\setlength{\tabcolsep}{3pt}
\begin{tabular}{lccc}
\toprule
\textbf{Model} & \textbf{CounterCurate} & \textbf{HallusionBench} & \textbf{SugarCrepe} \\
\midrule
\resultrow{Qwen2.5VL-7B~\cite{bai2025qwen2}}{61.4}{45.8}{85.2}
\resultrow{Qwen2.5VL-7B (CoT)}{60.4}{45.0}{83.2}
\resultrow{Video-R1 (SFT)}{60.6}{40.1}{83.3}
\resultrow{Video-R1 (RL)}{63.6}{41.4}{79.9}
\resultrow{ViGoRL (RL)~\cite{sarch2025grounded}}{67.9}{46.1}{79.1}
\resultrow{\model (Ours)}{\textbf{85.3}}{\textbf{49.7}}{\textbf{86.4}}
\bottomrule
\end{tabular}
\caption{Results on visual hallucination benchmarks. }
\label{tab:hallucination_results}
\vspace{-15pt}
\end{table}

\begin{table}[t]
\centering
\resizebox{\columnwidth}{!}{
\begin{tabular}{lccccccc}
\hline
Model  & Base & Common & Complex & Visual & Spatial & Long & Avg\\
\hline
Qwen2.5-VL-7B~\cite{bai2025qwen2}       & 4.0 & 3.3 & 2.0 & 0.0 & 0.7 & 1.3 & 1.9\\
Qwen2.5-VL-7B (CoT)      & 6.0 & 9.3 & 7.3 & 5.3 & 4.7 & 0.7 & 5.6\\
Video-R1~\cite{feng2025video}    & 16.7 & 11.3 & 18.0 & 8.0 & 5.3 & 0.7 & 10.0\\
ViGoRL ~\cite{sarch2025grounded}     & 14.7 & 11.3 & 16.0 & 8.7 & 4.0 & 0.7 & 9.2\\
\model (Ours)   & \textbf{24.7} & \textbf{18.0} & \textbf{27.3} & \textbf{8.7} & \textbf{8.7} & \textbf{0.7} & \textbf{14.7}\\
\hline
\end{tabular}
}
\caption{Results on EB-Alfred.}
\label{tab:ebalfred-results}
\vspace{-20pt}
\end{table}

\begin{table}[t]
\centering
\resizebox{\columnwidth}{!}{
\begin{tabular}{lccccccc}
\hline
Model  & Base & Common & Complex & Visual & Spatial & Long & Avg\\
\hline
Qwen2.5-VL-7B~\cite{bai2025qwen2}       & 23.3 & 5.3 & 10.7 & 6.7 & 8.0 & 0.0 & 9.0\\
Qwen2.5-VL-7B (CoT)       & 30.7 & 6.7 & 13.3 & 11.3 & 12.0 & 1.3 & 12.6\\
Video-R1~\cite{feng2025video}     & 46.7 & 6.0 & 15.3 & 9.3 & 16.7 & 3.3 & 16.2\\
ViGoRL~\cite{sarch2025grounded}     & \textbf{47.3} & 10.7 & 19.3 & \textbf{18.0} & \textbf{18.0} & 2.7 & 19.3\\
\model (Ours)   & 45.3 & \textbf{12.0}  & \textbf{24.0}  & 16.0 & 17.3 & \textbf{9.3} & \textbf{20.7}\\
\hline
\end{tabular}
}
\caption{Results on EB-Habitat.}
\label{tab:ebhabitat}
\vspace{-25pt}
\end{table}

\subsection{Hallucination Reasoning}
To evaluate the effectiveness of performing grounded reasoning, we compare \model against baseline approaches on three evaluation benchmarks: CounterCurate~\cite{zhang2024countercurate}, HallusionBench~\cite{guan2024hallusionbench} and SugarCrepe~\cite{hsieh2023sugarcrepe}. The results are summarized in Table~\ref{tab:hallucination_results}. These benchmarks are aimed at evaluating the capabilities of LMMs to reduce hallucination of visual concepts in their generated responses. First, we observe that our \model achieves a significant relative performance gain of more than 20\% over the Qwen2.5VL-7B base model on CounterCurate. We also see consistent improvements obtained by SOTA multimodal reasoning models, with Video-R1 outperforming the base model by $\sim$5\%. This suggests that reasoning is an important and necessary capability to reduce visual hallucination. We note that CounterCurate is considered to be an easier benchmark than HallusionBench and SugarCrepe, as it primarily evaluates on models' ability to differentiate between left/right and top/down. 

Despite using fewer training samples during both SFT and RL stages, \model outperforms Video-R1 by large margins on both HallusionBench and SugarCrepe. These results further emphasize the importance of performing grounded reasoning on alleviating hallucination in LMMs. In contrast, we see that Video-R1 actually performs worse than the Qwen2.5VL-7B base model. Furthermore, \model outperforms ViGoRL~\cite{sarch2025grounded}, which is also trained to perform grounded reasoning, hinting at the benefits of computing more informative rewards.

\subsection{Embodied AI}
To determine the importance of performing grounded reasoning for agentic foundation models, we also report the results of our evaluations on EmbodiedBench~\cite{yangembodiedbench} across high-level task planning and completion tasks in the Alfred and Habitat environments.
The results are summarized in Tables \ref{tab:ebalfred-results} and \ref{tab:ebhabitat}. 
As shown in Table \ref{tab:ebalfred-results}, the base Qwen2.5VL-7B model generalizes poorly to task planning for agentic task completion, despite its strong performance on standard visual question answering benchmarks. Notably, our \model improves significantly on the sub-category of ``complex'' tasks by over $25\%$ for task success rates to the base model. This result shows that the reason capability is particularly beneficial for planning solving complex multi-step tasks. Furthermore, our \model is able to outperform baselines on the ``visual'' and ``spatial'' subtasks. This superior performance strongly suggest that localizing referred objects explicitly in the reasoning trace helps LMMs to leverage the visual content far more effectively. 

In the Habitat evaluation environment, we observe similar trends (Table~\ref{tab:ebhabitat}) as those found on the Alfred benchmark. Here, CoT prompting is beneficial on high-level task planning even for the non-reasoning base model, as evidenced by the $\sim3.6\%$ performance gain obtained over the base Qwen2.5-VL-7B model on average. In particular, the results demonstrate the clear benefits of adding grounded reasoning for generalizing to complex visual understanding tasks, where our model outperforms the CoT-prompted Qwen2.5-VL-7B by approximately $7\%$.




\vspace{-15pt}
\begin{table}[th]
\centering
\footnotesize
\setlength{\tabcolsep}{3.5pt}
\renewcommand{\arraystretch}{1.15}
\begin{tabular}{l cccccc}
\toprule
 & 
      \multicolumn{4}{c}{\textbf{Libero}} \\
\cmidrule(lr){2-7}
\textbf{Model}   &
Object & Spatial & Goal & Long & 90 & Avg\\
\midrule
Qwen2.5-VL 7B~\cite{bai2025qwen2}          &  84.0 & \textbf{93.6} & 80.6 & 60.2 & 82.3 & 80.1\\
Qwen2.5-VL-7B-Instruct &  81.6 & 93.4 & 84.4 & 62.4 & 77.3 & 79.8\\
Qwen2.5-VL 7B (SFT)          &  88.0 & 91.0 & 84.0 & 66.1 & 83.3 & 82.4\\
Video-R1~\cite{feng2025video}                 & 89.2 & 93.0 & 89.6 & 65.6 & 80.1 & 83.5 \\
ViGoRL~\cite{sarch2025grounded}                 & 87.8 & 90.8 & 84.2 & 67.8 & 79.6 & 82.0 \\
\model (Ours)          &  \textbf{93.2} & 91.2 & \textbf{87.8} & \textbf{63.8} & \textbf{85.0} & \textbf{84.2}\\
\bottomrule
\end{tabular}
\caption{Success rates (\%) on the LIBERO~\cite{liu2023libero} continuous control robotics benchmark. When finetuned to predict continuous robot control actions, our MMRL approach outperforms baselines in terms of both performance and data efficiency. }
\label{tab:robotics-results}
\vspace{-30pt}
\end{table}

\subsection{Robotics evaluations}
Embodied decision-making often requires learning complex and generalizable robotic behaviors. Such agents typically require representations informed by world knowledge for perceptual grounding, planning, and control. Given that representations learned via grounded SFT embodied chain-of-thought~\cite{Zawalski24-ecot,chen25training,yang2025magma} have been shown to be conducive to more generalizable Vision-Language-Action models (VLAs)~\cite{zitkovich2023rt,kim2025openvla,open_x_embodiment_rt_x_2023}, we further evaluate our model on complex robotics tasks. We post-train \model and the baselines as VLAs to assess the transferability of their representations and their ability to fit complex, multimodal action distributions across multiple simulated benchmarks.

\vspace{-10pt}
\paragraph{Libero.}
We evaluate our method on the LIBERO simulation benchmark~\cite{liu2023libero}, which utilizes a Panda robot with delta end-effector control. Our evaluation follows two distinct protocols based on established task suites:
\begin{itemize}
    \item Specialized Suites: We first evaluate on four 10-task suites: LIBERO-Spatial, LIBERO-Object, LIBERO-Goal, and LIBERO-Long. For this setup, we train a single policy on the combined datasets from all four suites.
    \item Diverse Suite: We separately evaluate on the LIBERO-90 benchmark, which comprises 90 different tasks. We train a dedicated policy using only the LIBERO-90 dataset.
\end{itemize}

For all experiments, we use both third-person and wrist camera images as input. The training data is prepared by re-rendering images to 224x224, filtering unsuccessful demonstrations, and removing ``no-op'' actions, following \cite{pertsch25-fast}. We report the percentage of successful task completions in Table~\ref{tab:robotics-results}, averaged over 50 trials per task for the four specialized suites and 20 trials per task for LIBERO-90. We observe that our method outperforms the base model, the SFT models and even Video-R1, which was trained on 270k data samples, some of which contained spatial related data which should be beneficial for libero. In contrast, our multimodal reinforcement learning approach outperforms the baselines both in terms of performance as well as sample efficiency (260k vs 85k).

\begin{table}[t]
\centering
\footnotesize
\setlength{\tabcolsep}{1.5pt}
\resizebox{\textwidth}{!}{%
\begin{tabular}{lcccccccc}
\toprule
\textbf{Model} & \textbf{BLINK} & \textbf{MindCube-t} & \textbf{CV-Bench} &  \textbf{CounterCurate} & \textbf{HallusionBench} & \textbf{SugarCrepe} &
\textbf{EB-Alfred} &
\textbf{EB-Habitat}\\
\midrule
\ablateresultrow{\model (Ours)}{\textbf{57.6}}{\textbf{37.2}}{\textbf{79.5}}{81.9}{\textbf{49.1}}{\textbf{88.0}}{\textbf{14.3}}{\textbf{18.8}}
\ablateresultrow{-- RQ}{56.9}{36.3}{76.9}{81.9}{48.7}{87.5}{14.0}{17.8}
\ablateresultrow{-- VG}{55.8}{36.3}{77.5}{\textbf{82.5}}{48.0}{86.7}{13.1}{18.4}

\bottomrule
\end{tabular}
}
\caption{Ablation results on spatial reasoning benchmarks. Ablation results on visual hallucination benchmarks. We note that HallusionBench typically uses the since deprecated GPT-4 Turbo as a grader. Thus, we replace GPT-4 Turbo with GPT-4.1 as the grader which may lead to higher performances on average than before.}
\label{tab:all_ablation_results}
\vspace{-15pt}
\end{table}

\begin{figure}[t]
    \centering
    \begin{subfigure}{0.49\linewidth}
        \centering
        \includegraphics[width=\linewidth]{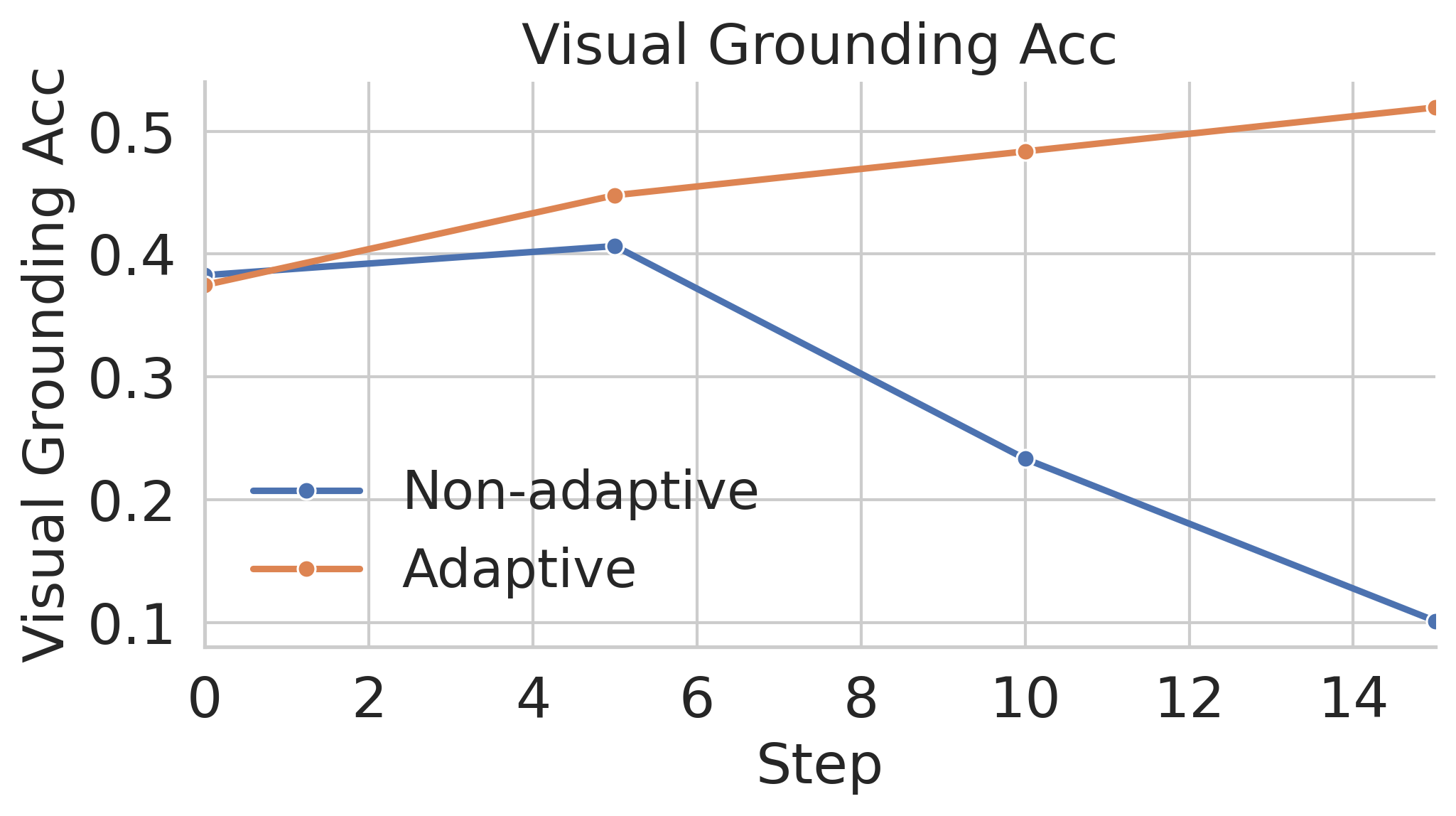}
        \caption{Visual grounding accuracy}
        \label{fig:reasoning_quality}
    \end{subfigure}
    \hfill
    \begin{subfigure}{0.49\linewidth}
        \centering
        \includegraphics[width=\linewidth]{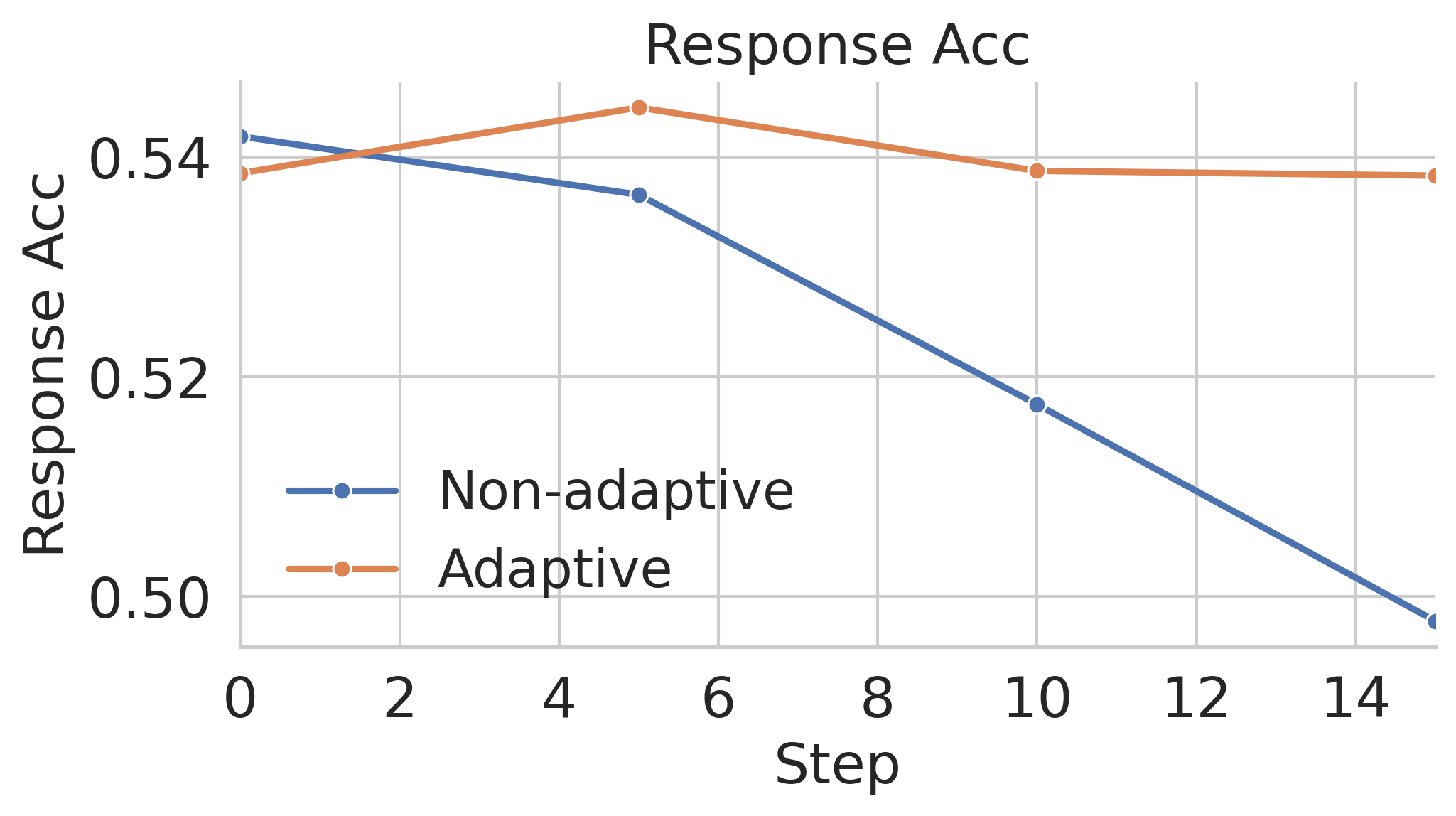}
        \caption{Response accuracy}
        \label{fig:second_metric}
    \end{subfigure}
    \vspace{-10pt}
    \caption{We run a small-scale comparison to ablate the effectiveness of \model (\textit{agentic}) compared to using only outcome rewards (\textit{non-adaptive}).}
    \label{fig:verifier_ablation}
    \vspace{-15pt}
\end{figure}

\subsection{Ablation Study}\label{sec:main-paper-ablation}
In this section, we ablate the effects of adding the different reward terms during the RL training stage in Table~\ref{tab:all_ablation_results} across the spatial reasoning, visual hallucination and embodied AI benchmarks. We conduct smaller-scale ablation experiments with a reduced training set and fewer numbers of training steps to analyze the reasoning quality (RQ) and visual grounding (VG) reward terms for MMRL on a small subset of the Pixmo-Count~\cite{deitke2024molmo} dataset combined with another subset of the Video-R1~\cite{feng2025video} image split. In total, our ablation RL training set numbers about 1.5K samples. Starting from the same SFT checkpoint finetuned with our curated data, we train three variants by removing the reasoning quality reward term and visual grounding reward terms in sequence. In general, we observe that adding the VG reward term is beneficial for improving performance across visual hallucination, spatial reasoning and embodied tasks. Additionally, we also see that adding the RQ reward term can also be helpful to further improve performance on such agentic benchmarks. Interestingly, using only the outcome reward term actually leads to the best performance on the CounterCurate benchmark in this ablation study. However, CounterCurate only evaluates on up/down and left/right questions is much less complex than HallusionBench and SugarCrepe. On the latter datasets which evaluates multiple axes of visual hallucination including removal and addition of objects, the results demonstrate that adding RQ and VG reward terms is generally beneficial.

\textbf{Adaptive verification helps mitigate OOD forgetting.} We analyze the importance of using our \model \space to compute more informative rewards for MMRL on a small subset of the Pixmo-Count~\cite{deitke2024molmo} dataset. Starting from the same SFT checkpoint finetuned with our curated data, we train two variants: one with \model (adaptive verifier) and one only using the outcome reward (non-adaptive verifier). We report the evaluation results on an out-of-domain (OOD) evaluation dataset of 1.5K samples sampled from Video-R1-CoT-165k in Figure~\ref{fig:verifier_ablation}.
We observe that, without verifiers of \model, the visual grounding accuracy drops rapidly. Importantly, training only on the outcome reward also leads to a performance drop on the validation set. This suggests that the aggregated final reward from our adaptive verifier helps reduce the likelihood of reward hacking.


\section{Conclusion}
In this work, we introduce \model, a novel  approach to perform MMRL by adaptively selecting different teacher models to compute dense rewards on a per-sample basis. It formulates MMRL into a multi-objective optimization problem by jointly rewarding outcome accuracy, spatiotemporal grounding, and reasoning quality, rather than relying on final answers alone. Importantly, we made two crucial observations. First, curating SFT grounded annotations is insufficient to train a grounded reasoning model. Second, our \model helps the base model to avoid reward-hacking by providing dense and robust rewards. Extensive experiments show that \model significantly outperforms prior work on a wide array of challenging agentic tasks, from spatial reasoning in images and videos to embodied AI and robotics. We hope that advances in task-specific models will improve \model’s reward signals, enabling the training of more effective and better-grounded multimodal agents in the future.


%
%
\bibliographystyle{splncs04}
\bibliography{main}

\clearpage
\appendix

In this supplemental, we provide the following additional material to the main paper:
\begin{enumerate}
    \item[A] Implementation details for training and evaluation
    \item[B] Agentic evaluation benchmark details
        \begin{enumerate}
            \item Spatial Reasoning Benchmarks
            \item Visual Hallucination Benchmarks
            \item Embodied AI and Robotics Benchmarks
        \end{enumerate}
    \item[C] Additional details on data curation pipeline
    \begin{enumerate}
            \item Image reasoning trace examples
            \item Video reasoning trace examples
        \end{enumerate}
    \item[D] Adaptive Verifier Scoring functions and prompts
        \begin{enumerate}
            \item Adaptive Selection of Scoring Functions
            \item Spatial Grounding Reward for Images
            \item Spatiotemporal Grounding Reward for Videos 
        \end{enumerate}
    \item[E] Full theoretical proof of Pareto-optimal solutions
    \item[F] Qualitative visualizations
\end{enumerate}

\section{Additional implementation details for training and evaluation}
We train our SFT and RL model variants using the AdamW optimizer with a 
learning rate of $1e^{-5}$, and 256 batch size for SFT and 56 for RL. We observe that training converges in around 1000 steps for SFT and 80 steps for RL. We implement our two-stage training pipeline using PyTorch with a combination of  8×H100 and 8xA100 40GB 
GPUs based on the Easy-R1 training framework~\cite{sheng2024hybridflow}. During evaluation, we use a maximum of 6144 or 8192 new tokens and a temperature value of 0.6.

\section{Agentic evaluation benchmarks}
In our experiment section, we mainly evaluate on several benchmarks that evaluate the agentic reasoning capabilities of our model on multiple tasks in diverse domains including visual hallucination, spatial intelligence and embodied AI. Here, we dive deeper into the setup and composition of these aforementioned benchmarks. 

\paragraph{Visual Hallucination.} Visual hallucination~\cite{liu2023hallusionbench,zhang2024countercurate} has been a critical limitation of using reasoning and non-reasoning based LMMs as AI agents. When a LMM hallucinate visual details such as non-existent objects or spurious relationships, it reduces the confidence of each subsequent decision since there might be limited relevance between perception, reasoning and planned actions. This is especially problematic in interactive settings where agents have to follow instructions to complete physical tasks or navigate GUI elements. Consequently, robustness to visual hallucination is a prerequisite for the deployment of multimodal AI agents. Thus, we evaluate the effectiveness of our model trained with the adaptive verifier to remain grounded in the visual scene. We use the accuracy metric for all visual hallucination benchmarks. 

The first benchmark is \textbf{CounterCurate}~\cite{zhang2024countercurate}, which contains hard positive and negative image-caption pairs for understanding ambiguous spatial relations. It is curated from the Flickr30K dataset and uses the GPT-4V and DALLE 3 models to generate semantically plausible but incorrect captions and images. We also evaluate on a visual-context  \textbf{HallusionBench}~\cite{liu2023hallusionbench}, which is aimed at evaluating the tendency of LMMs to be affected by language hallucination and visual illusion. It contains 346 images and 1129 questions that are authored by humans. Additionally, the questions are further categorized into those that are \emph{visual-dependent} and \emph{visual-supplement}. Finally, the \textbf{SugarCrepe}~\cite{hsieh2023sugarcrepe} benchmark is introduced to evaluate LMMs' capability to understand fine-grained vision-language compositionality. It contains 7512 questions to systematically cover 7 types of hard negatives including the replacement and addition of objects. 

\begin{figure*}[t]
\begin{AIBox}{Reward Verifier Instruction Template (Part 1)}
\small
\smallskip

You are an \textbf{agentic reward verifier} used in reinforcement learning. For each single sample, you must \textbf{plan which reward functions to run and in what order} to evaluate the predicted response. Each sample includes a question, a predicted response, a ground-truth answer, and visual inputs that may be a single image, a list of images, or one or more videos. Decide which functions to invoke so the predicted response is scored appropriately.

\textbf{REWARD FUNCTIONS YOU MAY CALL (names are case-sensitive):}
\begin{enumerate}
    \item \texttt{"visual\_grounding"}: evaluates whether the predicted 2D point coordinates accurately localize the objects or noun phrases referenced in the reasoning thoughts and final answer in natural images or video frames (e.g., everyday scenes, people, animals, COCO-like objects). The prediction is considered correct when the points correspond to the actual visual location of the target object, ensuring semantic consistency with the visual evidence rather than relying only on textual alignment.

    \item \texttt{"visual\_grounding\_synthetic"}: evaluates whether the predicted 2D point coordinates accurately localize referenced entities in synthetic imagery (e.g., math figures, charts/plots, tables, documents, UI screenshots, diagrams). The prediction is considered correct when the points correspond to the actual visual location of the target object, ensuring semantic consistency with the visual evidence rather than relying only on textual alignment.

    \item \texttt{"semantic\_accuracy"}: evaluates whether the predicted response preserves the same meaning as the ground-truth response, regardless of differences in phrasing or wording. This function measures the degree of semantic equivalence between the two answers rather than exact string matching.

    \item \texttt{"reasoning\_quality"}: evaluates the soundness and relevance of the reasoning process that leads from the question and visual inputs to the predicted response. Higher-quality reasoning is characterized by logical consistency, self-reflection, possible backtracking, correct use of the provided information, and minimal reliance on unsupported assumptions.

    \item \texttt{"exact\_match\_accuracy"}: evaluates whether the predicted response exactly matches the ground-truth response. This scoring function is used for tasks where the answer space is discrete and unambiguous, such as binary choice questions (e.g., yes/no), multiple-choice questions (where the answer must match one of the provided options), and object or event counting tasks (where the predicted numeric value must exactly equal the ground-truth count).

    \item \texttt{"relative\_tolerance\_accuracy"}: a soft-matching reward (1 or 0) that evaluates whether the predicted numeric response falls within an acceptable range of the ground-truth value. The prediction is considered correct (reward = 1) if it lies within $\pm 5\%$ of the ground-truth value; otherwise, the reward is 0. This metric is designed for tasks where answers are real-valued numbers and minor deviations from the exact value should be tolerated, such as measurements, ratios, probabilities, or other continuous-valued predictions.
\end{enumerate}

\end{AIBox}
\end{figure*}

\begin{figure*}[t]
\begin{AIBox}{Reward Verifier Instruction Template (Part 2)}
\small
\smallskip

\textbf{INPUTS (provided below):}
\begin{enumerate}
    \item \texttt{question}: string
    \item \texttt{gt\_response}: string
    \item \texttt{pred\_response}: string
    \item \texttt{reasoning\_response}: string
    \item \texttt{visual\_inputs}: JSON list like \texttt{["/abs/path/img\_001.png", "/abs/path/clip\_003.mp4", ...]}
\end{enumerate}

\textbf{AGENTIC PLANNING OBJECTIVES:}
\begin{enumerate}
    \item Infer the likely task type from the question/answers (e.g., counting, presence, attribute, spatial relation, open QA, binary QA, numeric estimation) without restating it.

    \item Choose an \textbf{ordered set} of reward functions to run for this sample.
    \begin{enumerate}
        \item Use \textbf{one} visual grounding function in almost all cases.
        \begin{enumerate}
            \item Prefer \texttt{"visual\_grounding"} for natural images/videos.
            \item Prefer \texttt{"visual\_grounding\_synthetic"} for synthetic/diagrammatic content (math figures, charts, tables, documents, UI, etc.).
        \end{enumerate}

        \item Augment with \texttt{"semantic\_accuracy"} when \texttt{gt\_response} and \texttt{pred\_response} are non-empty.

        \item Consider adding \texttt{"exact\_match\_accuracy"} for discrete answer spaces (yes/no, multiple choice, exact counts).

        \item Consider \texttt{"relative\_tolerance\_accuracy"} when the expected answer is continuous-valued and small numeric deviations should be tolerated.

        \item Optionally add \texttt{"reasoning\_quality"} for tasks where rationale fidelity matters.
    \end{enumerate}

    \item Provide a brief per-step rationale for why each chosen function is necessary \textbf{for this sample}.
\end{enumerate}

\end{AIBox}
\end{figure*}

\begin{figure*}[t]
\begin{AIBox}{Reward Verifier Instruction Template (Part 3)}
\small
\smallskip

\textbf{REQUIRED RULES:}
\begin{enumerate}
    \item Return ONLY a single JSON object with the exact schema in OUTPUT SCHEMA. No extra text.

    \item Never invent scoring function names. Use only the names listed above.

    \item Preconditions:
    \begin{enumerate}
        \item \texttt{"visual\_grounding"} OR \texttt{"visual\_grounding\_synthetic"}: \texttt{visual\_inputs} is a list with one or more non-empty paths. Choose exactly ONE of these two; prefer \texttt{"visual\_grounding"} for natural images/videos, and \texttt{"visual\_grounding\_synthetic"} for synthetic/diagrammatic content (figures, charts, tables, docs, UI).

        \item \texttt{"semantic\_accuracy"} requires non-empty predicted and ground-truth responses.

        \item \texttt{"exact\_match\_accuracy"}: task appears discrete/unambiguous (e.g., yes/no, multiple-choice label, exact count) AND both responses non-empty.

        \item \texttt{"relative\_tolerance\_accuracy"}: task appears numeric/continuous AND both responses non-empty.

        \item \texttt{"reasoning\_quality"}: \texttt{reasoning\_response} is non-empty (or present) and pertains to the prediction.
    \end{enumerate}

    \item Default policy (planning guidance):
    \begin{enumerate}
        \item If visual preconditions hold, include exactly one visual grounding function FIRST.

        \item Then, pick \textbf{exactly one} from:
        \texttt{["semantic\_accuracy", "exact\_match\_accuracy", "relative\_tolerance\_accuracy"]}
        based on the question type and expected response.
        \begin{enumerate}
            \item Use \texttt{"exact\_match\_accuracy"} for discrete answers (yes/no, MCQ label, exact integer count).
            \item Use \texttt{"relative\_tolerance\_accuracy"} for continuous/numeric estimates.
            \item Use \texttt{"semantic\_accuracy"} for open-form textual answers where meaning equivalence matters.
        \end{enumerate}
    \end{enumerate}

    \item Planning constraints:
    \begin{enumerate}
        \item \texttt{"functions"} must be an ordered list of unique names (no duplicates).
        \item Keep plans minimal: include only functions justified by this sample’s inputs and task.
    \end{enumerate}
\end{enumerate}

\end{AIBox}
\end{figure*}

\begin{figure*}[t]
\begin{AIBox}{Reward Verifier Instruction Template (Part 4)}
\small
\smallskip

\textbf{OUTPUT SCHEMA (strict):}

Return a JSON object with a single key \texttt{"functions"} whose value is an ordered array of 1..3 unique function names, each drawn from:
\texttt{["visual\_grounding", "visual\_grounding\_synthetic", "semantic\_accuracy", "exact\_match\_accuracy", "relative\_tolerance\_accuracy", "reasoning\_quality"]}.

\begin{Verbatim}[breaklines=true, breakanywhere=true, fontsize=\small]
{
  "functions": ["<chosen_name_1>", "<chosen_name_2>", "<chosen_name_3_optional>"]
}
\end{Verbatim}

Important: The array above is an \textbf{example of the format only}. Do \textbf{not} copy those names verbatim. Choose the functions based on your agentic planning for THIS sample.

\textbf{SAMPLE TO EVALUATE:}
\begin{enumerate}
    \item \texttt{question}: \{question\}
    \item \texttt{gt\_response}: \{gt\_response\}
    \item \texttt{pred\_response}: \{pred\_response\}
    \item \texttt{reasoning\_response}: \{reasoning\_response\}
    \item \texttt{visual\_inputs}: \{visual\_inputs\}
\end{enumerate}

\end{AIBox}
\end{figure*}

\begin{figure*}[t]
\centering
\begin{AIBox}{Image Instruction Template}
\small
\smallskip
\textbf{You are given:}
\begin{enumerate}
  \item The original question: \{question\}
  \item The associated image (if any).
  \item A model-generated answer: \{generated\_text\}
\end{enumerate}

\textbf{TASK}\\
Extract a concise set of referenced \textbf{entities, objects, and interactions}
that are \textbf{visibly present in the image}. Focus on the \textbf{main objects}
and their key visible attributes. When descriptive expressions are used, prefer
\textbf{expressive but non-redundant terms} that capture the important visible
features (e.g., ``seedling with roots'' instead of just ``seedling'').

\textbf{ENTITY RULES}
\begin{itemize}
  \item Entities should be \textbf{specific and expressive}, not overly vague.
  \item Avoid redundant variants that describe the same object with slightly
        different wording (e.g., do not include both ``seedling with roots''
        and ``early growth with roots and leaves''; pick the clearest and most
        representative one).
  \item Merge synonyms/aliases into one expressive label.
  \item Include distinct parts (like ``roots'', ``leaves'', ``soil'') only if
        they are visually important on their own, not just mentioned as part of
        a longer phrase.
  \item If more than 10 entities are visible, include only the
        \textbf{10 most salient} ones (salient = important, central, or
        repeatedly emphasized).
\end{itemize}

\textbf{INTERACTION RULES}
\begin{itemize}
  \item Only include interactions that are visually supported (e.g., arrows,
        measurable processes, progressive changes).
  \item Express them concisely but clearly (e.g., ``growth progression'',
        ``roots growing in soil'').
  \item Merge duplicates or semantically equivalent variants.
\end{itemize}

\textbf{OUTPUT FORMAT}\\
Return \textbf{only} this JSON object (no explanations, no extra text):
\begin{Verbatim}[breaklines=true, breakanywhere=true, fontsize=\small]
{
  "entities": [up to 10 expressive, non-redundant entities/objects visible in the image],
  "interactions": [unique interactions visually supported by the image, if any]
}
\end{Verbatim}
\end{AIBox}
\caption{Instruction template for extracting  entities and interactions from generated image-level rollouts.}
\label{fig:image-instruction-entity-extraction}
\end{figure*}

\begin{figure*}[t]
\centering
\begin{AIBox}{Video Instruction Template}
\small
\smallskip
\textbf{You are given:}
\begin{enumerate}
  \item The original question: \{question\}
  \item The associated video frames.
  \item A model-generated answer: \{generated\_text\}
\end{enumerate}

\textbf{FRAME TIMELINE}
\begin{itemize}
  \item You are given \{num\_frames\} frames with timestamps (seconds) in chronological order:\\
        \{formatted\_frame\_timestamps\}
  \item Use timestamps to infer ordering and durations. When reporting actions/interactions,
        include a time range like \(t_{\text{start}}\rightarrow t_{\text{end}}\) based on visible evidence.
  \item If frames are discontinuous or sparsely sampled, reason only over the provided times.
\end{itemize}

\textbf{TASK}\\
From the VIDEO FRAMES, extract a concise set of referenced
\textbf{entities, objects, actions, and interactions} that are
\textbf{visibly supported in the video} and \textbf{relevant to the question and answer}.
Emphasize \textbf{spatiotemporal evidence}: movements, changes over time, ordering
(before\(\rightarrow\)after), and persistent states across frames. If no video is present,
fall back to static visual evidence.

\textbf{ENTITY RULES}
\begin{itemize}
  \item Entities should be \textbf{specific and expressive}, not overly vague.
  \item Avoid redundant variants describing the same object; merge synonyms/aliases into one expressive label.
  \item Include distinct parts (e.g., ``roots'', ``leaves'', ``soil'') only if \textbf{visually important}
        on their own or central to the Q/A.
  \item If $>$10 entities are visible, include only the \textbf{10 most salient}
        (important, central, or repeatedly emphasized for the Q/A).
\end{itemize}

\textbf{ACTION \& INTERACTION RULES (SPATIOTEMPORAL)}
\begin{itemize}
  \item Prefer \textbf{motion/change verbs} and \textbf{state transitions}
        (e.g., ``seedling grows'', ``cup tilts'', ``person picks up book'').
  \item Capture \textbf{temporal structure} when visible: start/end states, increasing/decreasing,
        repetition (``x3''), or ordering (``A\(\rightarrow\)B'').
  \item Distinguish \textbf{object motion vs.\ camera motion}; do not attribute camera pans/zooms to objects.
\end{itemize}
\end{AIBox}
\caption{Instruction template for extracting entities and interactions from generated video-level rollouts (Part 1).}
\label{fig:video-instruction-entity-extraction}
\end{figure*}

\begin{figure*}[t]
\ContinuedFloat
\centering
\begin{AIBox}{Video Instruction Template (continued)}
\small
\smallskip

\textbf{ACTION \& INTERACTION RULES (SPATIOTEMPORAL) continued}
\begin{itemize}
  \item Only include interactions/actions that are \textbf{visually grounded across frames}
        and \textbf{support the Q/A}.
  \item Express concisely (e.g., ``roots growing in soil'',
        ``ball rolls left\(\rightarrow\)right and stops''); merge duplicates or equivalents.
\end{itemize}

\textbf{RELEVANCE FILTER}
\begin{itemize}
  \item Include only entities/actions/interactions that \textbf{help answer the question}
        or \textbf{support the given answer}; omit irrelevant background details.
\end{itemize}

\textbf{OUTPUT FORMAT}\\
Return \textbf{only} this JSON object (no explanations, no extra text):
\begin{Verbatim}[breaklines=true, breakanywhere=true, fontsize=\small]
{
  "entities": [up to 10 expressive, non-redundant entities/objects visible in the video],
  "actions": [concise motion/change events grounded across frames, if any.
              Encode their time ranges as (start_time=Xs, end_time=Ys)],
  "interactions": [unique spatiotemporal relations expressed in concise natural English,
                   if any]
}
\end{Verbatim}
\end{AIBox}
\caption{Instruction template for extracting entities and interactions from generated video-level rollouts (Part 2).}
\end{figure*}

\begin{figure*}[ht]
    \centering
    \begin{subfigure}{0.48\textwidth}
        \centering
        \includegraphics[width=\linewidth]{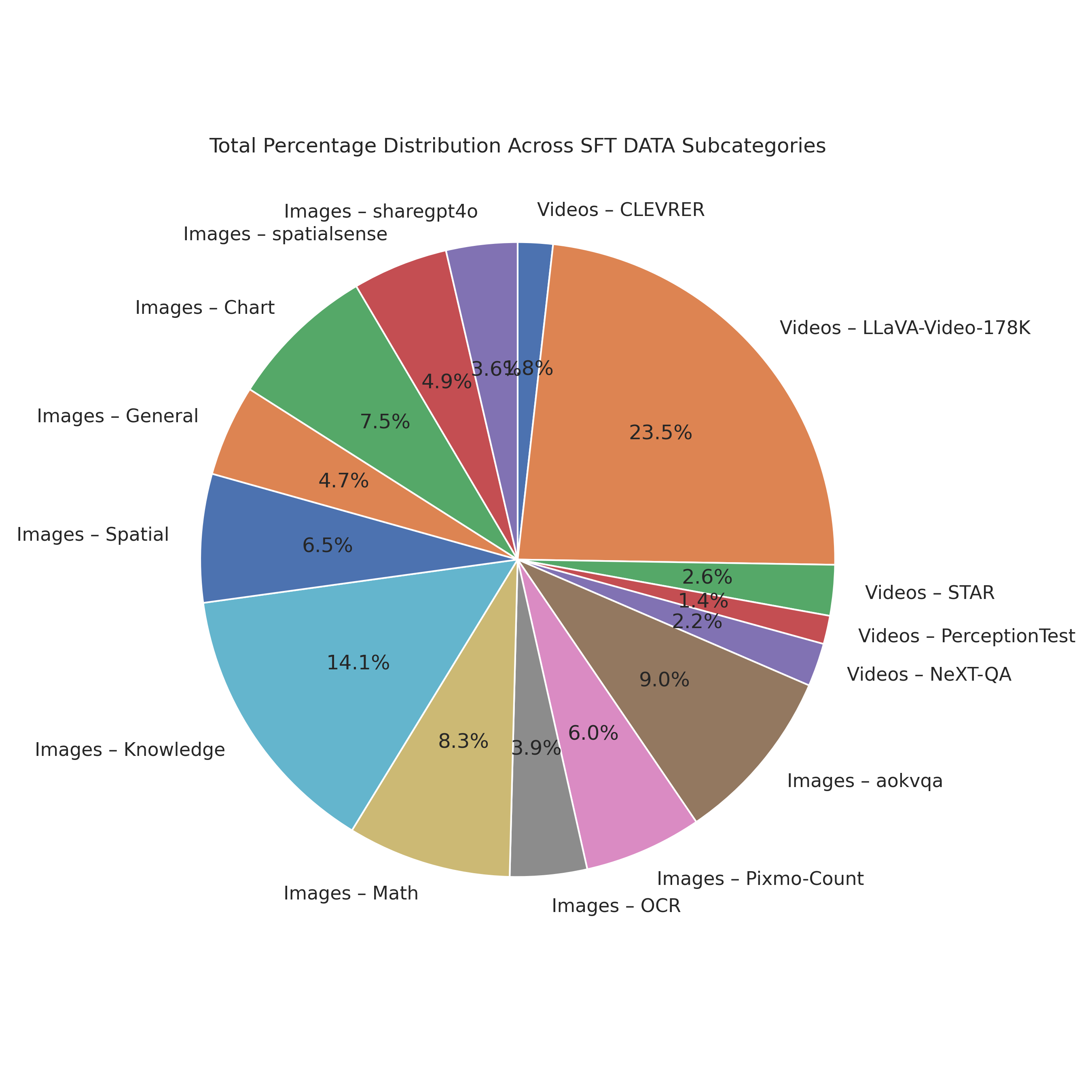}
        \label{fig:sft-data-mix}
    \end{subfigure}
    \hfill
    \begin{subfigure}{0.48\textwidth}
        \centering
        \includegraphics[width=\linewidth]{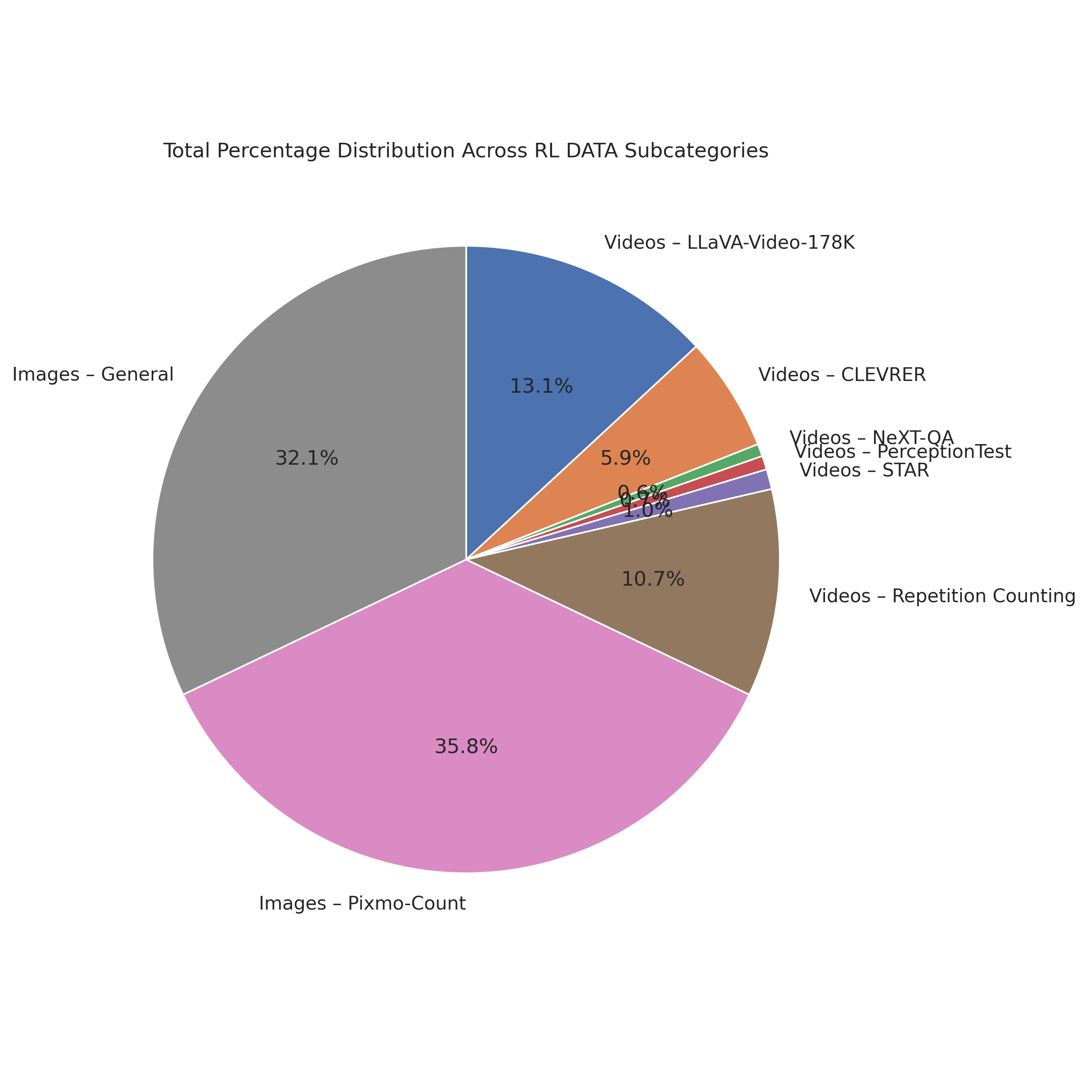}
        \label{fig:rl-data-mix}
    \end{subfigure}
    \caption{\textbf{Training data mixtures.} We use around 85K and 4.6K SFT and RL training samples, respectively.}
    \label{fig:combined-data-mix}
\end{figure*}

\paragraph{Spatial Reasoning.} Besides robustness to visual hallucination, spatial reasoning is a core capability for multimodal agents to interact effectively in the physical world. AI agents have to understand where objects are and the relationships between them in space to generate feasible and accurate responses. We report performance using the accuracy metric since all benchmarks are based on multiple choice-questions.

First, we begin by evaluating on the \textbf{BLINK}~\cite{fu2024blink} benchmark, which contains around 3.8K questions that are generated from 7.3K images. This benchmark sources questions from 14 classic computer vision tasks including depth estimation and visual similarity. Second, we use the tiny version of the MindCube~\cite{yin2025spatial} dataset, which is one of the most recent spatial reasoning benchmarks. It is intended to test the ability of LMMs to construct internal spatial mental models from a few visual perspectives. The full benchmark contains around 21K questions, which are categorized into the ``Around'', ``Among'' and ``Rotation'' types. Due to the scale of the data, we evaluate on the tiny version, which is a smaller but representative subset spanning the same spatial settings. Finally, we also evaluate on the Cambrian Vision-Centric Benchmark (CV-Bench)~\cite{tong2024cambrian}, which consists of about 2.6K manually-inspected evaluation samples. These questions evaluate LMMs' 2D and 3D spatial understanding capability. 

\paragraph{Embodied AI and Robotics.}
Last but not least, we also evaluate on agentic task planning and completion using two popular benchmarks. The first benchmark, EmbodiedBench~\cite{yangembodiedbench}, is a dataset that spans multiple environments such as AI2-THOR and Habitat 2.0. In our experiments, we mainly focus on the high-level task completion setting. The EB-Alfred sub-task evaluates an agent’s ability to compose and execute high-level skills like pick up, open, slice, and find objects over diverse household tasks. It contains approximately 300 test samples evenly split across the categories of base, commonsense, complex instructions, visual appearance, spatial awareness and long-horizon. In contrast, the EB-Habitat evaluation task focuses on high-level object rearrangement based on language queries. Finally, we also evaluate on the Libero evaluation suite~\cite{liu2023libero}, which is a dataset for language-based robot manipulation. It mainly consists of Libero-Spatial, Libero-Object, Libero-Goal and Libero-100. The Libero-100 subset is further split into Libero-90 which contains 90 short-horizon tasks and 10 long-horizon tasks in Libero-Long.

\section{Additional details on data curation}
\label{sec:suppl-data-curation}
In this section, we present more specific details about our entire data curation pipeline starting from rollout generation to the final verification process.

\subsection{Data preparation}
Our data preparation process consists of multiple steps. In Section 3, we briefly describe our process to extract relevant objects, interactions and events in images and videos. In this section, we augment that description by providing the specific prompts that we used to query GPT-4o to extract the relevant meta-information from the question, ground-truth response and original generated rollouts. The prompt template used for image samples (Figure~\ref{fig:image-instruction-entity-extraction}) is almost identical to that used for video samples (Figure~\ref{fig:video-instruction-entity-extraction}), except that we also extract information about actions that span multiple frames. In both cases, the LM will return a JSON dictionary where the extracted information 00is stored in the relevant key-value pairs. Based on the extracted objects and interactions, we use the Molmo-7B model to generate their 2D point coordinates. Note that Molmo-7B normalizes the 2D coordinates between 0 and 100 while the base Qwen2.5-VL 7B model represents pixel coordinates in absolute units. After rescaling the generated pixel coordinates, we overlay these coordinates on images and sampled video frames. 

\begin{figure*}[t]
\centering
\begin{AIBox}{Point–Object Extraction Prompt}
\small
You will be given a multimodal reasoning trace that includes text with
references to visual observations and 2D points in \((x, y)\) format.

Your task is to extract all 2D point coordinates in the text, and for each
one, identify the most semantically relevant noun phrase, object mention, or
referring expression that is closely associated with that point. This may
include specific items seen in the image, categories, or localized objects
mentioned in the same or nearby sentences.

Return your output as a list of mappings in the following format:

\begin{verbatim}
[
  {"point": [x, y], "object": "most relevant noun phrase or object"},
  ...
]
\end{verbatim}

Do not include any other text.

\textbf{Guidelines:}
\begin{itemize}
  \item Only consider 2D points mentioned explicitly in the text
        (e.g., (190, 247)).
  \item Prefer noun phrases that refer to concrete objects or regions
        localized in the image (e.g., ``notebooks'', ``mechanical pencils'').
  \item If multiple possible objects are nearby in the text, pick the one
        most likely referred to by the point based on sentence structure,
        context, or positional cues.
  \item Do not hallucinate new coordinates or object names not mentioned in
        the original text.
  \item For each point, ensure the associated object is specific and
        descriptive, avoiding vague terms like ``thing'' or ``item''.
        It should never be null.
\end{itemize}

Here is the reasoning trace: \{text\}
\end{AIBox}
\caption{We use this prompt template to query a LM to extract the generated 2D points and their corresponding object or noun phrase.}
\label{fig:point-extraction-prompt}
\end{figure*}

\begin{figure*}[t]
\centering
\begin{AIBox}{Video Point--Event Extraction Prompt (Part I)}
\footnotesize
You will be given a multimodal \textbf{video} reasoning trace that may include:
\begin{itemize}
  \item Explicit 2D points written as \texttt{(x, y)} or \texttt{[x, y]}
  \item Temporal anchors like \texttt{[F<id> @ t=<seconds>s]}
  \item Temporal lists like \texttt{[F1 @ t=0.00s, F2 @ t=1.66s, F3 @ t=3.31s, ...]}
  \item Temporal spans like \texttt{[F<start>-F<end> @ t=<start>-<end>s]}
  \item Mentions of objects (nouns/referring expressions), actions/events,
        and single-frame observations
\end{itemize}

\textbf{YOUR TASK}\\
Return \textbf{STRICT JSON} with three top-level arrays:
\texttt{"observations"}, \texttt{"events"}, \texttt{"points"}.
Map extracted items as follows (this mapping is REQUIRED):
\begin{itemize}
  \item \textbf{"points"}: ONLY items that are a point+anchor
        \textbf{as an exact substring} in the text $\to$
        \texttt{"(x, y) in [Fk @ t=Ts]"} or \texttt{"[Fk @ t=Ts] at (x, y)"}.
  \item \textbf{"events"}: ONLY multi-frame spans
        \texttt{"[Fstart-Fend @ t=Tstart-Tend]"} or temporal lists
        \texttt{"[F1 @ t=T1, F2 @ t=T2, ...]"}.
  \item \textbf{"observations"}: ONLY standalone single anchors
        \texttt{"[Fk @ t=Ts]"} (no point in same substring) or standalone
        points \texttt{"(x, y)"} / \texttt{"[x, y]"} (no timestamp in same substring).
\end{itemize}

\textbf{ANCHOR TEXT --- MUST BE AN EXACT SUBSTRING (VERBATIM)}\\
For every record, return \texttt{"anchor\_text"} as a single string that is an
exact substring copied verbatim from the provided text. Do \emph{not} rewrite
or synthesize connectors.

Allowed formats for \texttt{"anchor\_text"} (choose exactly one per record):
\begin{itemize}
  \item \texttt{"(x, y) in [Fk @ t=Ts]"} or \texttt{"[Fk @ t=Ts] at (x, y)"} (goes to \texttt{"points"})
  \item \texttt{"[Fstart-Fend @ t=Tstart-Tend]"} (goes to \texttt{"events"})
  \item \texttt{"[F1 @ t=T1, F2 @ t=T2, F3 @ t=T3, ...]"} (goes to \texttt{"events"})
  \item \texttt{"[Fk @ t=Ts]"} (goes to \texttt{"observations"})
  \item \texttt{"(x, y)"} or \texttt{"[x, y]"} (goes to \texttt{"observations"})
\end{itemize}

\textbf{CRITICAL EXAMPLE (exactness requirement)}\\
Source text: For example, ``In \texttt{[F17 @ t=19.21s]}, pallbearers at
coordinates \texttt{(140, 679)}, \texttt{(290, 554)}, and \texttt{(654, 544)}
maintain a steady stance.''

Valid \texttt{"anchor\_text"}:
\begin{itemize}
  \item \texttt{"[F17 @ t=19.21s]"} (observation)
  \item \texttt{"(140, 679)"} (observation)
  \item \texttt{"(290, 554)"} (observation)
  \item \texttt{"(654, 544)"} (observation)
\end{itemize}
\end{AIBox}
\caption{We use this prompt template to query a LM to extract the generated 2D points and their corresponding object or noun phrase, as well as timestamps for video reasoning traces.}
\label{fig:video-point-extraction-prompt}
\end{figure*}

\begin{figure*}[t]
\ContinuedFloat
\centering
\begin{AIBox}{Video Point--Event Extraction Prompt (Part II)}
\footnotesize

\textbf{STRUCTURES TO RETURN}

\textbf{1) OBSERVATIONS (ONLY Format 4 or 5)}
\begin{verbatim}
{
  "anchor_text": "<exact substring (Format 4 or 5)>",
  "frame": Fk_or_null,
  "time_s": Ts_or_null,
  "description": "<short phrase nearest to anchor_text describing /
  what is visible/occurring>",
  "object": "<noun phrase/object or null>",
  "points": [[x, y], ...]
}
\end{verbatim}

\textbf{2) EVENTS (ONLY Format 2 or 3)}
\begin{verbatim}
{
  "anchor_text": "<exact substring (Format 2 or 3)>",
  "event": "<action/event phrase>",
  "frames": [F_start, F_end],
  "times_s": [t_start, t_end],
  "points": [[x, y], ...]
}
\end{verbatim}

\textbf{3) POINTS (ONLY Format 1)}
\begin{verbatim}
{
  "anchor_text": "<exact substring (Format 1)>",
  "point": [x, y],
  "object": "<most relevant noun phrase or object>",
  "frame": Fk,
  "time_s": Ts
}
\end{verbatim}

\textbf{RETURN FORMAT (STRICT JSON ONLY; NO EXTRA TEXT)}\\
Return only valid JSON with exactly three top-level keys:
\texttt{"observations"}, \texttt{"events"}, and \texttt{"points"}.

\textbf{GUIDELINES}
\begin{itemize}
  \item Do not output any explanation before or after the JSON.
  \item Each record must use exactly one allowed \texttt{"anchor\_text"} format.
  \item \texttt{"anchor\_text"} must be copied verbatim from the source text.
  \item Put an item in \texttt{"points"} only if the point and timestamp appear in
        the same exact substring.
  \item Put temporal spans and temporal lists only in \texttt{"events"}.
  \item Put standalone timestamps or standalone points only in \texttt{"observations"}.
\end{itemize}

Here is the reasoning trace: \{text\}

\end{AIBox}
\caption{We use this prompt template to query a LM to extract the generated 2D points and their corresponding object or noun phrase, as well as timestamps for video reasoning traces (continued).}
\end{figure*}

\subsection{Reasoning trace generation}
To generate rollouts, we primarily use the GLM-4.1V 9B reasoning model. As mentioned in the main paper, we prompt the model to ground its reasoning in the original visual input, while using the 2D point coordinates to disambiguate objects in images and to specify frame-level or multi-frame events in videos. For each sample, we generate 8 rollouts with a temperature of 0.6 and a maximum of 6144 new tokens. We use the prompt templates in Figure~\ref{fig:point-extraction-prompt} and \ref{fig:video-point-extraction-prompt} to extract generated 2D points, timestamps and their corresponding descriptions in images and videos, respectively.

\begin{figure*}[t]
\centering
\begin{AIBox}{Video Scoring Prompt (Part 1/2)}
\small
\medskip
You are an information extraction model. Read the reasoning text and extract
three categories of temporally grounded facts. Follow the rules exactly.

\textbf{DEFINITIONS}

\textbf{1) spatiotemporal}
\begin{itemize}
  \item Target: every 2D point tag of the form \texttt{<points ...>...</points>}.
  \item For each \texttt{<points>} tag, produce ONE item with:
  \begin{itemize}
    \item \texttt{"frame"}: the nearest explicitly stated single frame number
          (case-insensitive \texttt{"frame N"}) in the same sentence or immediately
          preceding clause, if present; else null.
    \item \texttt{"time"}: the nearest explicitly stated single absolute time in
          seconds (e.g., \texttt{"36.31 seconds"}, \texttt{"t=36.31s"}, \texttt{"36.31s"})
          in the same sentence or immediately preceding clause, if present; else null.
    \item \texttt{"object"}: the object name from the tag’s inner text if non-empty,
          otherwise from its \texttt{alt} attribute. This has to be an object and not
          timestamps or frame numbers.
    \item \texttt{"x"}: the x coordinate from the tag.
    \item \texttt{"y"}: the y coordinate from the tag.
  \end{itemize}
  \item Do NOT infer values; only use explicitly stated numbers. If both a frame and a
        time are given, include both; otherwise set the missing field to null.
  \item If multiple \texttt{<points>} tags are present near the same mention, output
        one item per tag.
\end{itemize}

\textbf{2) frame-level temporal}
\begin{itemize}
  \item Target: any mention tied to a SINGLE frame (e.g., ``frame 6'') and/or a
        SINGLE time (e.g., ``4.47 seconds''). IMPORTANT: This is similar to the
        spatiotemporal case but used when no \texttt{<points>} tag is present.
        No 2D coordinates are involved here and there should not be any overlap
        with the extracted spatiotemporal points.
  \item Output an item with:
  \begin{itemize}
    \item \texttt{"frame"}: that single frame number, or null if none.
    \item \texttt{"time"}: that single absolute time in seconds, or null if none.
    \item \texttt{"description"}: an EXACT substring copied verbatim from the input
          text that describes what happens at that frame/time. Do NOT paraphrase or
          invent wording; choose the minimal contiguous snippet that fully states
          the observation.
  \end{itemize}
  \item Do NOT include \texttt{<points>} tags themselves inside \texttt{"description"}
        unless they are part of the original wording you copy.
\end{itemize}
\end{AIBox}
\caption{Video scoring prompt used to extract referenced objects and actions in videos (Part 1).}
\label{fig:video-scoring-prompt-part1}
\end{figure*}

\begin{figure*}[t]
\centering
\begin{AIBox}{Video Scoring Prompt (Part 2/2)}
\small

\textbf{3) segment-level temporal}
\begin{itemize}
  \item Target: any mention that spans MULTIPLE frames (e.g., ``frames 1--6'',
        ``frame 7 through frame 20'') and/or a TIME RANGE (e.g., ``4.47 to 16.11
        seconds'').
  \item Output an item with:
  \begin{itemize}
    \item \texttt{"start\_frame"}: first frame number if given, else null.
    \item \texttt{"end\_frame"}: last frame number if given, else null.
    \item \texttt{"start\_time"}: start time in seconds if given, else null.
    \item \texttt{"end\_time"}: end time in seconds if given, else null.
    \item \texttt{"description"}: an EXACT substring copied verbatim from the input
          that describes the segment-level event. Do NOT paraphrase.
  \end{itemize}
  \item Open-ended ranges like ``frame 25 onwards'' are allowed: fill known
        \texttt{start\_*} and set unknown \texttt{end\_*} fields to null.
\end{itemize}

\textbf{NORMALIZATION \& MATCHING RULES}
\begin{itemize}
  \item Frames: integers only. Match case-insensitively for the word
        ``frame'' (e.g., ``Frame 6'', ``frame 6'').
  \item Times: floats in seconds; accept forms like ``36.31 seconds'',
        ``t=36.31s'', ``36.31s''. Output them as numbers (no units) rounded
        to two decimals.
  \item Ranges: recognize hyphen/en dash (``1-6'', ``1--6''), ``from \dots\ to \dots'',
        and ``through''. For ``onwards'', only the \texttt{start\_*} is known.
  \item Association for spatiotemporal points: prefer the closest explicit single
        frame/time in the same sentence or immediately preceding clause. If only
        a range is present and no single value is explicitly tied to the point,
        set both \texttt{"frame"} and \texttt{"time"} to null (do NOT invent).
  \item Deduplicate identical items across lists.
  \item Use null where a field cannot be filled \textbf{without explicit evidence}.
  \item Do NOT hallucinate any text. For the \texttt{"description"} fields in (2)
        and (3), the value MUST be an exact substring found in the input text.
\end{itemize}

\textbf{OUTPUT FORMAT (STRICT)}\\
Return ONLY a JSON object with exactly these keys:
\begin{itemize}
  \item \texttt{"spatiotemporal"}: a list of objects of the form
        \texttt{{"frame": <int|null>, "time": <float|null>, "object": <string>, "x": <string>, "y": <string>}}
  \item \texttt{"frame\_level\_temporal"}: a list of objects of the form
        \texttt{{"frame": <int|null>, "time": <float|null>, "description": <string>}}
  \item \texttt{"segment\_level\_temporal"}: a list of objects of the form
        \texttt{{"start\_frame": <int|null>, "end\_frame": <int|null>, "start\_time": <float|null>, "end\_time": <float|null>, "description": <string>}}
\end{itemize}

If a category has no items, return an empty list for that key.
\smallskip
\textbf{TEXT}\\
\{reasoning\_text\}
\end{AIBox}
\caption{We use this template to prompt the teacher model to extract referenced objects and actions in videos.}
\label{fig:video-scoring-prompt}
\end{figure*}

\begin{figure*}[t]
\centering
\begin{AIBox}{Video frame-level accuracy}
\small
\smallskip
You will receive ONE image (a single video frame) and a short description.

\textbf{TASK}\\
Score how well the description matches ONLY the visible content of the image.
Ignore unverifiable timing info (timestamps/frame indices) in the text; treat
them as neutral metadata. Focus on objects, attributes, and spatial relations
that can be visually verified in this single frame.

\textbf{SCORING (0.0 or 1.0)}
\begin{itemize}
  \item \textbf{1.0} All key claims are clearly supported; no contradictions.
  \item \textbf{0.0} Contradicted or refers to things not visible in the frame.
\end{itemize}

\textbf{RULES}
\begin{itemize}
  \item Judge ONLY this single image (no assumptions from outside the frame).
  \item If a relation like ``near the closet door'' is claimed, require both
        entities to be visible and the relation plausible in the image.
  \item Be conservative when evidence is unclear; do not infer beyond what is visible.
\end{itemize}

\textbf{OUTPUT (JSON only, no extra text)}
\begin{verbatim}
{"score": <float 0 - 1>}
\end{verbatim}

\textbf{DESCRIPTION}\\
"\{description\}"
\end{AIBox}
\caption{We use this template to prompt the teacher model to assign a binary score based on the relevance of the generated frame-level description and the visual content in the corresponding video frame.}
\label{fig:video-frame-acc}
\end{figure*}

\begin{figure*}[t]
\centering
\begin{AIBox}{Video segment-level accuracy}
\small
\smallskip
You will receive an ORDERED sequence of images (consecutive video frames) and a
short event/action description.

\textbf{TASK}\\
Score how well the description matches ONLY what is visually supported across the
entire frame sequence. Evaluate both the action semantics and the temporal extent
(start$\rightarrow$middle$\rightarrow$end). Assume that the provided frames correspond
to the mentioned timestamps and rely primarily on visible evidence and ordering
of frames.

\textbf{INPUTS}
\begin{itemize}
  \item \textbf{FRAMES}: An ordered list of frames F1..FN. Assume they are evenly spaced.
  \item \textbf{DESCRIPTION}: ``\{description\}''
\end{itemize}

\textbf{EVALUATION GUIDELINES}
\begin{itemize}
  \item \textbf{Visual Evidence}: Are the required entities/objects present? Is the claimed
        action (e.g., ``cutting the cake'') visibly happening?
  \item \textbf{Temporal Progression}: Do frames show a plausible sequence for that action
        (setup $\rightarrow$ interaction $\rightarrow$ outcome)? Look for state changes
        (e.g., knife touches cake, slices appear).
  \item \textbf{Extent \& Alignment}: Does the action substantially occur within the claimed
        start/end span? Small off-by-one frame or slight timing drift is minor; clear
        mismatches are penalized.
  \item \textbf{Consistency}: Penalize claims contradicted by any clear frame (e.g., no knife
        ever shown, or subject doing a different action).
  \item \textbf{Single-Sequence Scope}: Judge ONLY these frames; don’t assume anything
        outside the provided sequence.
\end{itemize}

\textbf{SCORING (0.0 or 1.0)}
\begin{itemize}
  \item \textbf{1.0 -- Fully supported}: the action is clearly visible and progresses as
        described, with onset/offset aligned within a small tolerance; no contradictions.
  \item \textbf{0.0 -- Contradicted}: action not happening, key objects missing, or description
        clearly mismatched.
\end{itemize}

\textbf{OUTPUT (JSON only, no extra text)}
\begin{verbatim}
{"score": <float 0 - 1>}
\end{verbatim}
\end{AIBox}
\caption{We use this template to prompt the teacher model to assign a binary score based on the relevance of the generated segment-level description and the visual content in the corresponding video segment.}
\label{fig:video-seg-acc}
\end{figure*}

\subsection{Verification and filtering}
Although the curation pipeline produces grounded reasoning traces, many rollouts remain unreliable, with only $\sim3.1\%$ being usable. To address this, \model is also used to score and filter rollouts, discarding samples whose best score falls below a threshold while still enforcing outcome accuracy. In practice, we use a threshold value of 0.7.  We further clean and standardize extracted 2D points and temporal phrases. This filtering ensures that the SFT dataset contains primarily high-quality, visually grounded, and semantically accurate reasoning examples.

\subsection{SFT and RL training data mixtures}
We plot the distributions of our final curated training data mixtures in Figure~\ref{fig:combined-data-mix}. The final set of SFT training samples numbers approximately 85K images and videos . During the RL stage, we use a much smaller subset of around 4.5K images and videos from the Video-R1-260K full split. Our curated datasets contains both video and image samples. The selected videos consist of diverse open-domain clips depicting everyday scenarios, aimed at improving temporal understanding and event reasoning. On the image side, we include general visual QA for basic perception, alongside more task-specific skills
such as interpreting scientific figures for quantitative reasoning, reading embedded text via OCR tasks and performing visual commonsense reasoning. Finally, we incorporate some spatial reasoning samples to help the model learn to reason not just about what is present, but where and how it is arranged.

\section{Adaptive verifier scoring functions}
The evaluation processes of the visual grounding accuracy of generated responses for images and videos differ slightly. For images, we observe that using a simple regex expression is sufficient to extract 2D points and their corresponding objects. In video reasoning traces, we prompt GPT-4o to extract 2D points as well as timestamps using the template in Figure~\ref{fig:video-scoring-prompt}. For each trace, the returned dictionary contains the following lists:
\begin{enumerate}
    \item spatiotemporal points with 2D point coordinates and relevant frames

    \item frame-level descriptions with relevant frames

    \item segment-level descriptions with the start and end frames.
\end{enumerate}
The generated spatiotemporal points are scored using the same process as images, where we describe how we use the visual grounding teacher models in Section 3. In the case of a frame-level description, we first encode its corresponding frame before querying the GLM-4.5V model to assign a binary score using the template in Figure~\ref{fig:video-frame-acc}. Similarly, we query the GLM-4.5V model to evaluate the visual-semantic accuracy between the encoded set of frames and the generated description using the template in Figure~\ref{fig:video-seg-acc}. Finally, we compute a final visual grounding accuracy score by first averaging the scores within each category before averaging over them.

\clearpage

\setcounter{theorem}{0}
\setcounter{definition}{0}
\setcounter{corollary}{0}
\setcounter{assumption}{0}
\setcounter{remark}{0}

\section{Theoretical analysis}\label{sect-theory}
To justify the intuition that combining complementary but potentially weak reward teachers may train stronger model with reinforcement finetuning, we provide a theoretical analysis on learning with multiple reward estimators. Inspired by \cite{zhao2024pareto}, we show that even noisy reward signals in aggregation can guide the policy towards global Pareto-optimal solutions.

Consider any prompt/input to the model, each possible answer/action $a \in \mathcal{A}$ can be evaluated by $m$ \textit{oracle} reward components
$R_1(a), \dots, R_m(a)$, estimating its quality from different aspects (e.g. grounding, reasoning, etc.). The multiple rewards can arbitrarily correlate with each other, forming a complementary set. It is possible that one answer/action is better than another in one reward, but worse in another. Following multi-objective optimization practice~\cite{miettinen1999nonlinear}, we define our learning goal with Pareto optimality:

\begin{definition}[$\delta$-Pareto Domination]
For $\delta > 0$, we say $a' \succ_\delta a$ if
$R_i(a') \ge R_i(a) + \delta$ for all $i=1,\dots,m$.
\end{definition}

\begin{definition}[$\delta$-Pareto Optimality]
For $\delta > 0$, the set of globally $\delta$-Pareto-optimal actions 
\[
P_\delta := \{a \in \mathcal{A}: \nexists a' \in \mathcal{A} \text{ s.t. } a' \succ_\delta a\}.
\]
\end{definition}

However, in practice we do not have access to the true rewards $R_i(a)$'s, but instead weak estimators $\hat R_i(a)$ which can be noisy, biased and correlated with each other. We assume the estimated reward follows the general form below:

\begin{assumption}
For any $a$, the estimated reward has the form 
\[
\hat R_i(a) = f(R_i(a)) + \varepsilon_i(a),
\]
where $f$ is a monotonically increasing function, and the stochastic term $\{\varepsilon_i(a)\}_{i=1}^m$ is independent and
sub-Gaussian with:
\[
\mathbb{E}[\varepsilon_i(a)] = \mu, \quad
\mathbb{E}[e^{\lambda (\varepsilon_i(a)-\mu)}] \le e^{\sigma^2 \lambda^2/2},
\ \forall \lambda \in \mathbb{R}.
\]
\end{assumption}

The $\sigma$-sub-Gaussian assumption is a generic form constraining the error, which is easily satisfied by Hoeffding's lemma when the rewards are bounded. 

In each RFT step, we optimize $\hat R(a) = \sum_{i=1}^m w_i \hat R_i(a)$. Denote $w_{\min} = \min_i w_i$ and $w_{\max} = \max_i w_i$. Even if we don't directly optimize the oracle rewards, we can show probability bound on the estimated $\delta$-Pareto optimality.

\begin{lemma}
\label{lem:pairwise}
Fix two actions $a,a'$, with $a' \succ_\delta a$.. Let $\delta_f := \inf_x [f(x+\delta) - f(x)] > 0$ be the margin in the nonlinear transformed space. Then we have
\[
\mathbb{P}\big(\hat R(a) \ge \hat R(a')\big)
\le
\exp\!\left(
- \frac{\delta_f^2}{4\sigma^2}
\, m \, \frac{w_{\min}^2}{w_{\max}^2}
\right).
\]
\end{lemma}

\begin{proof}
If $a' \succ_\delta a$, then $R_i(a') - R_i(a) \ge \delta$. By monotonicity, the deterministic difference is:
\[
\sum_{i=1}^m w_i (f(R_i(a')) - f(R_i(a))) \ge \delta_f \sum_i w_i \ge \delta_f m w_{\min}.
\]
Let $Z = \sum_i w_i(\varepsilon_i(a')-\varepsilon_i(a))$. The expectation
\begin{align}
    \mathbb{E}[Z] &= \sum_i w_i(\mathbb{E}[\varepsilon_i(a')]-\mathbb{E}[\varepsilon_i(a)]) \\
    &= \sum_i w_i(\mu-\mu]) = 0
\end{align}
Using the sub-Gaussian property, we have $\forall \lambda \in \mathbb{R}$,
\begin{align*}
    \mathbb{E}[e^{\lambda Z}] &= \mathbb{E}[e^{\lambda \sum_i w_i(\varepsilon_i(a')-\varepsilon_i(a))}] \\
&=  \prod_i^m  \mathbb{E}[e^{\lambda w_i(\varepsilon_i(a)-\mu)}] \mathbb{E}[e^{\lambda w_i(\varepsilon_i(a')-\mu)}] \\
&\le e^{m w_{\max}^2 \sigma^2 \lambda^2}.
\end{align*}
Thus $Z$ is sub-Gaussian with proxy variance bounded by
$2m\sigma^2 w_{\max}^2$, and we have
\begin{align*}
    \mathbb{P}\big(\hat R(a) \ge \hat R(a')\big) &= \mathbb{P}\big(\hat R(a) - \hat R(a') \ge 0\big) \\
    &\le \mathbb{P}\big(Z \le -\delta_f m w_{\min}\big) \\
    &\le \exp\!\left( - \frac{(\delta_f m w_{\min})^2}{2 (2 m w_{\max}^2 \sigma^2)} \right) = \exp\!\left( - \frac{\delta_f^2}{4\sigma^2} \cdot m \cdot \frac{w_{\min}^2}{w_{\max}^2} \right).
\end{align*}
\end{proof}

Next, we study optimization of $\hat R(a)$ over a finite sample group $\mathbb{G} = \{a_1, \dots, a_n\}$ drawn from policy $\pi$.

\begin{lemma}[Batch-Level Pareto Preservation]
\label{lem:batch}
Let $\mathbb{G}$ be a group of $n$ actions and $\hat a = \arg\max_{a \in \mathbb{G}} \hat R(a)$. Then:
\[
\mathbb{P}\big(\exists a' \in \mathbb{G} \text{ s.t. } a' \succ_\delta \hat a\big) \le (n-1)\exp(-C \cdot m),
\]
where $C := \frac{\delta_f^2}{4\sigma^2} \frac{w_{\min}^2}{w_{\max}^2}$.
\end{lemma}

\begin{proof}
By Lemma~\ref{lem:pairwise}, each pair $(\hat a, a')$ for $a' \ne a$ violates the
ordering with probability less than $\exp\!\left(
- \frac{\delta_f^2}{4\sigma^2}
\, m \, \frac{w_{\min}^2}{w_{\max}^2}
\right)$. Simply applying a union bound
over the $n-1$ other samples yields the probability bound
\[
(n-1)\exp\!\left(
- \frac{\delta_f^2}{4\sigma^2}
\, m \, \frac{w_{\min}^2}{w_{\max}^2}
\right) = (n-1)\exp(-C \cdot m).
\]
\end{proof}

With the above lemmas, we can guarantee global $\delta$-Pareto optimality over the entire possible action space.

\begin{theorem}[Global Pareto Guarantee]
\label{thm:global}
Let $\pi$ be the sampling policy with probability coverage $\beta = \pi(P_\delta)$. For a group $\mathbb{G}$ of $n$ i.i.d. samples, the selected action $\hat a = \arg\max_{a \in \mathbb{G}} \hat R(a)$ satisfies:
\[
\mathbb{P}(\hat a \in P_\delta) \ge \big(1-(1-\beta)^n\big) \left[ 1 - (n-1)e^{-C \cdot m} \right].
\]
\end{theorem}

\begin{proof}
With probability $1-(1-\beta)^n$, the batch contains at least one element
of $P_\delta$.  Conditioned on that event, Lemma~\ref{lem:batch} bounds the probability that the contained optimal action is not selected.  Multiplying the two probabilities yields the stated bound.
\end{proof}

\noindent The theorem shows that as the number of rewards $m$ increases, we can approximate global Pareto optimality even with weak estimators.

\begin{corollary}[Additional Correctness Reward]
\label{cor:gated-pareto}
Let $R_0(a)$ denote a dominant outcome \emph{correctness} reward and $R_1(a),\dots,R_m(a)$ the reasoning rewards. For a threshold $\tau$ and margin $\gamma > 0$, define the gated scalarization:
\[
\hat R_\tau(a) =
\begin{cases}
\hat R_0(a), & \hat R_0(a) < \tau,\\[3pt]
w_0 \hat R_0(a) + \sum_{i=1}^m w_i \hat R_i(a), & \hat R_0(a) \ge \tau,
\end{cases}
\]
where all $w_i > 0$. Assume that for all actions, the expected estimated correctness reward $E[\hat R_0(a)] = f(R_0(a)) + \mu$ satisfies a gap condition: $|(f(R_0(a)) + \mu) - \tau| > \gamma$.

Let $P_{\delta,\tau}$ be the set of actions that are (i) correctly gated ($f(R_0(a)) + \mu \ge \tau + \gamma$) and (ii) $\delta$-Pareto-optimal in the reasoning rewards. Then, when optimizing $\hat R_\tau(a)$ over a batch $\mathbb{G}$ of size $n$:
\[
\mathbb{P}(\hat a \in P_{\delta,\tau})
\ge
\big(1-(1-\beta)^n\big)
\big(1 - 2n e^{-\frac{\gamma^2}{2\sigma^2}}\big)
\!\left[
1 - \frac{n-1}{e^{C \cdot (m+1)}}
\right],
\]
where $\beta = \pi(P_{\delta,\tau})$ and $C$ is defined as in Lemma~\ref{lem:pairwise} but for $m+1$ dimensions.
\end{corollary}

\begin{proof}
We perform similar analysis by conditioning the joint events. The probability that $\mathbb{G}$ contains at least one action $a^* \in P_{\delta,\tau}$ is $1-(1-\beta)^n$. 

We then bound the gating error. An action is misclassified if $\hat R_0(a) \ge \tau$ while $f(R_0(a)) + \mu < \tau$, or vice versa. Given the gap $\gamma$, this requires $|\varepsilon_0(a) - \mu| > \gamma$. By the $\sigma$-sub-Gaussian assumption:
\[
\mathbb{P}(|\varepsilon_0(a) - \mu| > \gamma) \le 2e^{-\frac{\gamma^2}{2\sigma^2}}.
\]
Applying a union bound over $n$ actions, the probability that all actions are correctly gated is at least $1 - 2n e^{-\gamma^2/2\sigma^2}$. 

Conditioned on correct gating, the optimization for high-correctness actions involves $m+1$ rewards (including $R_0$). Since $f$ is monotonic and the bias $\mu$ cancels in pairwise comparisons (as shown in Lemma~\ref{lem:pairwise}), the ranking of the weighted sum $\hat R(a)$ preserves the Pareto ordering. The probability that the estimated optimizer $\hat a$ is indeed the Pareto-optimal choice follows from Lemma~\ref{lem:batch} adapted for $m+1$ estimators:
\[
1 - (n-1)e^{-C(m+1)}.
\]
Combining these independent components yields the stated lower bound.
\end{proof}

\begin{remark}[GRPO Training]
\label{rem:grpo}
In Group Relative Policy Optimization (GRPO), the advantage is computed as $A_i = (\hat R_i - \text{mean}(\mathbb{G}))/\text{std}(\mathbb{G})$. Since the mean and standard deviation are constants for a fixed batch $\mathbb{G}$, this transformation is strictly monotonic. Consequently, $\arg\max A_i = \arg\max \hat R_i$. The Pareto-optimality guarantees established for aggregated rewards apply directly to GRPO, ensuring that Pareto-optimal responses receive the highest relative reinforcement.
\end{remark}

\begin{remark}[Optimal Sample Size]
\label{rem:optimal-batch}
The lower bound in Theorem~\ref{thm:global},
\[
H(n)=\underbrace{\big(1-(1-\beta)^n\big)}_{\text{exploration}} \cdot 
\underbrace{\left[1-\frac{n-1}{e^{C m}}\right]}_{\text{exploitation}}
\]
shows an inherent exploration--exploitation tradeoff. Larger $n$ increases the probability of sampling an optimal action (exploration) but also increases the probability that noise causes a misranking (exploitation error). 

By setting the derivative $H'(n) \approx 0$, we find the approximate optimal batch size:
\[
n^\star \approx \frac{1}{\beta}\log\!\big(1+\beta e^{C m}\big).
\]
This implies that when the model is weak ($\beta$ is small), a larger batch size $n$ is required to ensure $P_\delta$ is sampled (exploration). As the model improves ($\beta$ increases), $n$ can be reduced to minimize ranking errors (exploitation).
\end{remark}

\section{Qualitative visualizations}
In this section, we provide some qualitative visualizations of our curated SFT reasoning traces. For image samples, we highlight any generated 2D points in the reasoning traces and also overlay them as red circles onto the images for visibility. We provide both original and the overlaid images for comparison but we note that we do not actually use the overlaid images during the SFT stage.

\begin{figure*}[t]
\centering
\begin{AIBox}{Image Example 1}
\begin{minipage}{0.48\textwidth}
    \centering
    \includegraphics[width=0.9\textwidth]{figures/im_visuals/a2b63f5f-7497-5d63-bcf0-184e412b4120.jpg}
\end{minipage}
\hfill
\begin{minipage}{0.48\textwidth}
    \centering
    \includegraphics[width=0.9\textwidth]{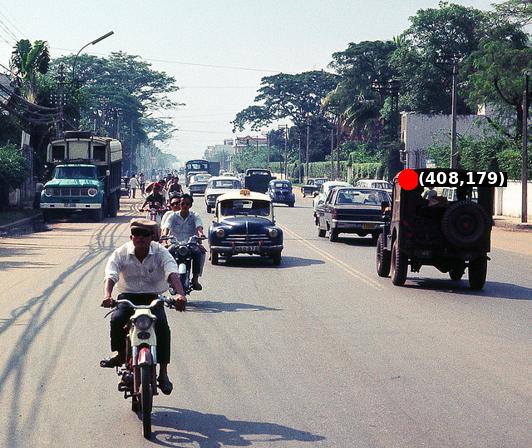}
\end{minipage}

\vspace{1em}

\small
\textbf{Question:}

\begin{lstlisting}[breaklines=true, basicstyle=\small\ttfamily, breakindent=0pt, escapechar=@]
Is there a vehicle with a spare tire visible on its exterior? A. No. B. Yes.
\end{lstlisting}

\vspace{0.3em}
\textbf{Response:}

\begin{lstlisting}[breaklines=true, basicstyle=\small\ttfamily, breakindent=0pt, escapechar=@]
<think>Let me analyze the image carefully while identifying the vehicle with a spare tire.

Scanning the scene, I need to look for vehicles that show a spare tire on their exterior. I'll move systematically through the image looking at each vehicle.

I notice a dark-colored jeep or utility vehicle on the right side of the image, which appears to have a spare tire mounted on its exterior. Let me focus on the exact location of this vehicle.

Looking at the right side of the image, I can see a vehicle that appears to be a jeep or utility vehicle. Upon closer inspection, I can confirm there's a spare tire mounted on the back of this vehicle, which is clearly visible. This appears to be the vehicle with the spare tire on its exterior.
\end{lstlisting}
\end{AIBox}
\caption{Image example 1}
\label{fig:annotated_a2b63f5f-7497-5d63-bcf0-184e412b4120_part1}
\end{figure*}

\begin{figure*}[t]
\centering
\begin{AIBox}{Image Example 1 (continued)}
\small
\textbf{Response (continued):}

\begin{lstlisting}[breaklines=true, basicstyle=\small\ttfamily, breakindent=0pt, escapechar=@]
I observe a jeep type vehicle at approximately @\textcolor{red}{<points x1="408" y1="179" alt="jeep or utility vehicle">jeep or utility vehicle</points>}@ with a spare tire visible at its rear. The spare tire appears to be mounted on the exterior of this vehicle, making it clearly visible.

Therefore, the answer should be B. Yes. There is a vehicle with a spare tire visible on its exterior.</think><answer>I examine the scene systematically from left to right. As I scan the right side of the image toward the middle-right area, I notice a dark-colored utility vehicle that appears to be a jeep type. On its exterior, I can clearly see a spare tire mounted on the back of the vehicle. This spare tire is visible and prominently displayed on the side of the vehicle's exterior structure. The vehicle is positioned near the right edge of the scene, traveling on the right side of the road. I confirm this is indeed a vehicle with a spare tire visible on its exterior.

<|begin_of_box|>B. Yes.<|end_of_box|>
\end{lstlisting}
\end{AIBox}
\caption{Image example 1 (continued)}
\label{fig:annotated_a2b63f5f-7497-5d63-bcf0-184e412b4120_part2}
\end{figure*}

\begin{figure*}[t]
\centering
\begin{AIBox}{Image Example 2 (Part 1)}
\begin{minipage}{0.48\textwidth}
    \centering
    \includegraphics[width=\textwidth]{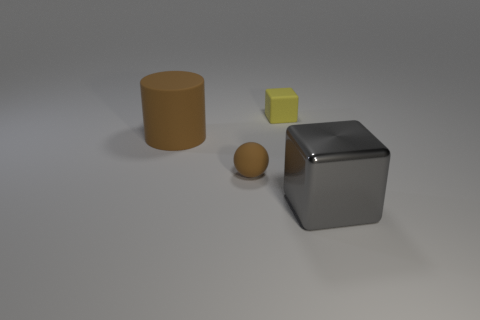}
\end{minipage}
\hfill
\begin{minipage}{0.48\textwidth}
    \centering
    \includegraphics[width=\textwidth]{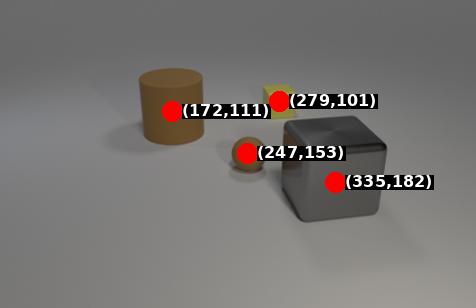}
\end{minipage}

\vspace{1em}

\small
\textbf{Question:}

\begin{lstlisting}[breaklines=true, basicstyle=\small\ttfamily, breakindent=0pt, escapechar=@]
How many objects are there in total?
\end{lstlisting}

\vspace{0.3em}
\textbf{Response:}

\begin{lstlisting}[breaklines=true, basicstyle=\small\ttfamily, breakindent=0pt, escapechar=@]
<think>Looking at the original image, I need to count and identify all objects. Starting from the left side, I see a large brown cylinder positioned at around the left-center area. Moving rightward, there is a small brown sphere near the center of the image. Continuing further right, there is a gray metallic cube. At the top right area, I see a small yellow cube. 

To confirm with spatial awareness, I scan from left to right: the large brown cylinder is at @\textcolor{red}{<points x1="172" y1="111" alt="large brown cylinder">large brown cylinder</points>}@ area, the small brown sphere is positioned closer to the center around @\textcolor{red}{<points x2="247" y2="153" alt="small brown sphere">small brown sphere</points>}@, the large gray metallic cube is at @\textcolor{red}{<points x3="335" y3="182" alt="large gray metallic cube">large gray metallic cube</points>}@, and the small yellow cube is positioned at the top right area at about @\textcolor{red}{<points x4="279" y4="101" alt="small yellow cube">small yellow cube</points>}@.
\end{lstlisting}
\end{AIBox}
\caption{Image example 2 (part 1)}
\label{fig:annotated_ee16c240-ac4d-554d-9ad5-607de4b981d4_part1}
\end{figure*}

\begin{figure*}[t]
\centering
\begin{AIBox}{Image Example 2 (Part 2)}
\small
\textbf{Response (continued):}

\begin{lstlisting}[breaklines=true, basicstyle=\small\ttfamily, breakindent=0pt, escapechar=@]
Counting these objects, there are 4 distinct objects total. I observe a large brown cylinder at @\textcolor{red}{<points x1="172" y1="111" alt="large brown cylinder">large brown cylinder</points>}@, a small brown sphere at @\textcolor{red}{<points x2="247" y2="153" alt="small brown sphere">small brown sphere</points>}@, a large gray metallic cube at @\textcolor{red}{<points x3="335" y3="182" alt="large gray metallic cube">large gray metallic cube</points>}@, and a small yellow cube at @\textcolor{red}{<points x4="279" y4="101" alt="small yellow cube">small yellow cube</points>}@. Therefore, there are 4 objects in total.</think><answer>Looking at the scene, I observe a large brown cylinder at position @\textcolor{red}{<points x1="172" y1="111" alt="large brown cylinder">large brown cylinder</points>}@ on the left side. Moving rightward toward the center, I notice a small brown sphere at @\textcolor{red}{<points x2="247" y2="153" alt="small brown sphere">small brown sphere</points>}@. Continuing further right, there is a large gray metallic cube at position @\textcolor{red}{<points x3="335" y3="182" alt="large gray metallic cube">large gray metallic cube</points>}@. In the upper right area of the image, I see a small yellow cube at position @\textcolor{red}{<points x4="279" y4="101" alt="small yellow cube">small yellow cube</points>}@. Counting all these distinct objects, there are 4 objects in total.
<|begin_of_box|>4<|end_of_box|></answer>
\end{lstlisting}
\end{AIBox}
\caption{Image example 2 (part 2)}
\label{fig:annotated_ee16c240-ac4d-554d-9ad5-607de4b981d4_part2}
\end{figure*}

\begin{figure*}[t]
\centering
\begin{AIBox}{Image Example 3}
\begin{minipage}{0.48\textwidth}
    \centering
    \includegraphics[width=0.3\textwidth]{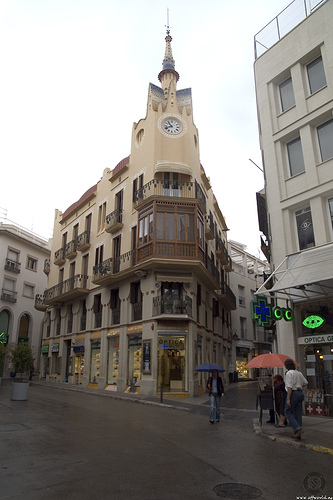}
\end{minipage}
\hfill
\begin{minipage}{0.48\textwidth}
    \centering
    \includegraphics[width=0.3\textwidth]{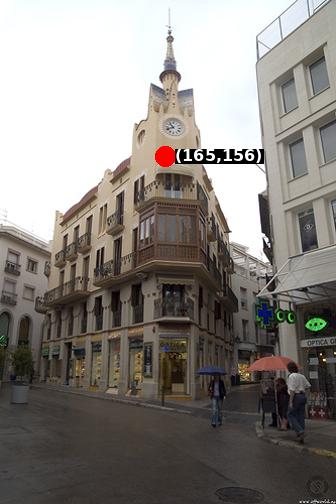}
\end{minipage}

\vspace{0.5em}

\small
\textbf{Question:}

\begin{lstlisting}[breaklines=true, basicstyle=\small\ttfamily, breakindent=0pt, escapechar=@]
What is at the top of this structure in the middle of the city square? A. church tower. B. advertisement. C. latitude. D. optiplex.
\end{lstlisting}

\vspace{0.3em}
\textbf{Response:}

\begin{lstlisting}[breaklines=true, basicstyle=\small\ttfamily, breakindent=0pt, escapechar=@]
<think>Let me analyze the image focusing on the structure in the middle of the city square. I need to identify what's at the top of this structure.

I observe a prominent building in the center of the image which appears to be the structure in question. Looking at the top of this building:
- There's a spire-like structure at the top
- There's a clock visible on what appears to be a tower
- The building has architectural features typical of a church tower or bell tower

Looking more carefully at the top of this central building specifically, I can see that there's a tall spire or tower element which has a distinctive shape typical of church architecture. The building appears to be historic with ornate architectural details around what would typically be a church tower or clock tower.

Now let me consider the answer choices:
A. church tower - This seems most likely given the architectural style with the spire, the historic nature of the building, and the prominent tower
B. advertisement - There doesn't appear to be an advertisement at the top, rather architectural elements
C. latitude - This is a directional coordinate measurement, not something you'd find at the top of a structure
D. optiplex - This appears to be related to computer equipment, not something that would be at the top of a building
\end{lstlisting}
\end{AIBox}
\caption{Image example 3 (part 1)}
\label{fig:annotated_1f82389b-dc45-5176-8278-3167dad5016e_part1}
\end{figure*}

\begin{figure*}[t]
\centering
\begin{AIBox}{Image Example 3 (cont.)}
\small
\textbf{Response (continued):}

\begin{lstlisting}[breaklines=true, basicstyle=\small\ttfamily, breakindent=0pt, escapechar=@]
Based on the visual evidence of the spire/tower structure with architectural details like what appears to be a clock, this fits the description of a church tower rather than an advertisement, latitude (which is not applicable), or optiplex (which is irrelevant).

I observe a distinctive spire and tower at the top of the central building around @\textcolor{red}{<points x1="165" y1="156" alt="church tower">church tower</points>}@ area, which has architectural characteristics typical of a church tower.</think><answer>Looking at the central building in the city square, I can see it has a distinctive tower structure. Examining the top portion of this central structure, I observe a prominent spire-like element at the very top, approximately at the middle-upper section of the beige building with the clock. The architectural style includes ornate details and the tower appears to be an elevated structure with a pointed or tapered top that's characteristic of religious or civic architecture. Given the building's historic appearance with its detailed facade, corner tower, and clock, the structure at the top is most likely a church tower. None of the other options (advertisement, latitude, or optiplex) would naturally be found as a prominent feature at the top of this architectural structure.

My answer is <|begin_of_box|>A<|end_of_box|>.</answer>
\end{lstlisting}
\end{AIBox}
\caption{Image example 3 (part 2)}
\label{fig:annotated_1f82389b-dc45-5176-8278-3167dad5016e_part2}
\end{figure*}

\begin{figure*}[t]
\centering
\begin{AIBox}{Image Example 4}
\begin{minipage}{0.48\textwidth}
    \centering
    \includegraphics[width=\textwidth]{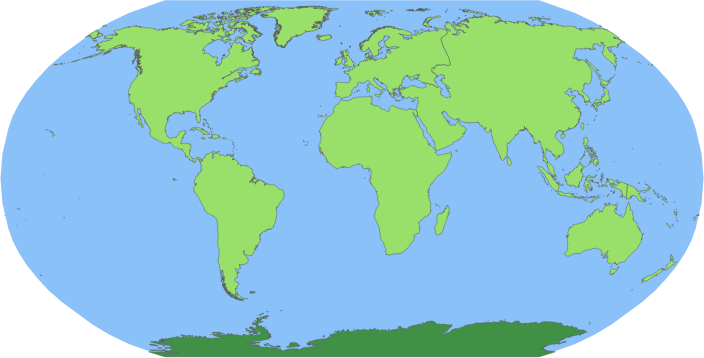}
\end{minipage}
\hfill
\begin{minipage}{0.48\textwidth}
    \centering
    \includegraphics[width=\textwidth]{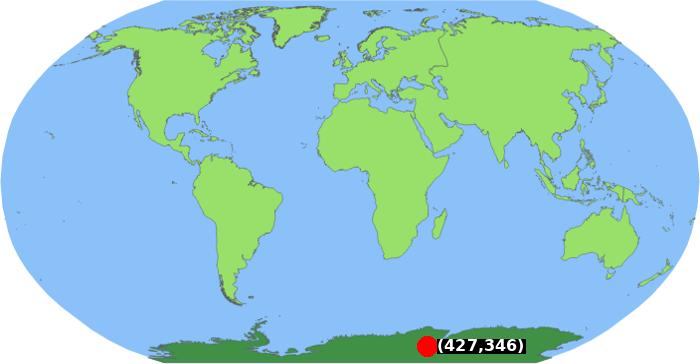}
\end{minipage}

\vspace{0.5em}

\small
\textbf{Question:}

\begin{lstlisting}[breaklines=true, basicstyle=\small\ttfamily, breakindent=0pt, escapechar=@]
Which continent is being indicated in the highlighted area? A. Australia. B. Antarctica. C. Africa.
\end{lstlisting}

\vspace{0.3em}
\textbf{Response:}

\begin{lstlisting}[breaklines=true, basicstyle=\small\ttfamily, breakindent=0pt, escapechar=@]
<think>Let me analyze the original image carefully. I'm looking for the highlighted area that indicates a specific continent.

Looking at the original image (first image), I can see a world map with blue oceans and green land masses. There's a distinct green landmass at the bottom of the image that appears to be completely surrounded by water. This is typically Antarctica on a world map.

Now I need to focus on that highlighted area specifically. Looking at the bottom portion of the original image, there is indeed a green landmass at the southernmost part of the map. It appears completely isolated by water, which is characteristic of Antarctica.

Looking at the second image (with coordinate overlays), I can see a coordinate at @\textcolor{red}{<points x1="427" y1="346" alt="green landmass at the southernmost part of the map">green landmass at the southernmost part of the map</points>}@ which appears to be pointing to the green area at the very bottom of the map, confirming that it is the southernmost continent.

Based on the position at the bottom of the world map, completely separated by water from all other continents, I can identify this as Antarctica.</think><answer>I examine the original world map, focusing on the lower portion where there is a distinct green landmass completely surrounded by water at the southernmost region of the map.
\end{lstlisting}
\end{AIBox}
\caption{Image example 4 (part 1)}
\label{fig:annotated_5df51f14-2f5e-56ee-a3bf-39aaa555d791_part1}
\end{figure*}

\begin{figure*}[t]
\centering
\begin{AIBox}{Image Example 4 (continued)}
\small
\textbf{Response (continued):}

\begin{lstlisting}[breaklines=true, basicstyle=\small\ttfamily, breakindent=0pt, escapechar=@]
Looking at this highlighted area, I note that it's positioned at the bottom of the map, separated by ocean from all other continents. This corresponds to the continent located at the South Pole. This geographic characteristic - isolated at the southernmost point of the map, completely surrounded by water - clearly indicates Antarctica. The distinctive positioning and separation from other landmasses confirms this is the proper identification.
<|begin_of_box|>B.<|end_of_box|></answer>
\end{lstlisting}
\end{AIBox}
\caption{Image example 4 (part 2)}
\label{fig:annotated_5df51f14-2f5e-56ee-a3bf-39aaa555d791_part2}
\end{figure*}

\begin{figure*}[t]
\centering
\begin{AIBox}{Video Example 1}
\begin{minipage}{\textwidth}
    \centering
    \includegraphics[width=\textwidth]{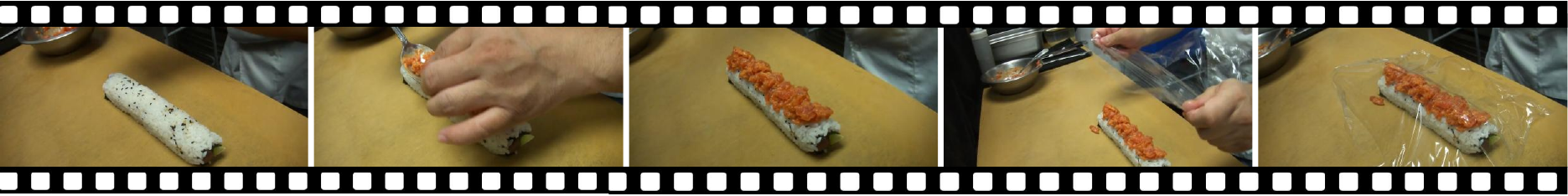}
\end{minipage}

\vspace{0.5em}

\small
\textbf{Question:}

\begin{lstlisting}[breaklines=true, basicstyle=\small\ttfamily, breakindent=0pt, escapechar=@]
What does the person do after topping the sushi roll with the orange fish mixture?
\end{lstlisting}

\vspace{0.3em}
\textbf{Response:}

\begin{lstlisting}[breaklines=true, basicstyle=\small\ttfamily, breakindent=0pt, escapechar=@]
<think>Got it, let's try to figure out what the person does after topping the sushi roll with the orange fish mixture. First, I need to look at the sequence of frames. The question is about what happens immediately after the topping step. 

Looking at the frames, after adding the orange fish mixture (which seems to be done in frames where the topping is visible), the next action would be wrapping it. Let's check the frames where the person is handling the sushi roll. After the topping, the person likely wraps it in plastic wrap or uses a bamboo mat. 

Looking at the frames, after the topping (which is done in frames like frame 2, frame 3, frame 4, frame 5, etc.), the next step would be wrapping with plastic wrap. Let's check frame 24 where the person is placing plastic wrap over the sushi roll. Wait, let's check the sequence. The person first tops with the orange mixture, then wraps in plastic wrap, then rolls with bamboo mat. So the action after topping would be wrapping the sushi roll in plastic wrap. Let me confirm the frames. After the topping is done (frames where the orange mixture is on top), the next action is wrapping in plastic wrap, then rolling with the bamboo mat. So the spatiotemporal reasoning would involve tracking the sushi roll's state. After being topped with the orange fish mixture, the person wraps the sushi roll in plastic wrap. Then uses the bamboo mat to roll it.
\end{lstlisting}
\end{AIBox}
\caption{Video example 1} 
\label{fig:video_example_1_part1}
\end{figure*}

\begin{figure*}[t]
\centering
\begin{AIBox}{Video Example 1 (continued)}
\small
\textbf{Response (continued):}

\begin{lstlisting}[breaklines=true, basicstyle=\small\ttfamily, breakindent=0pt, escapechar=@]
Now, to make sure I'm citing the correct frames. Let's check the frames where the plastic wrap is used. Frame 24 shows the person placing plastic wrap over the sushi roll, and then frames 25 to 27 show wrapping with plastic wrap. Then frames 28 to 30 show using the bamboo mat to roll the sushi. So the action after topping is wrapping the sushi roll in plastic wrap.</think><answer>To determine what the person does after topping the sushi roll with the orange fish mixture, we analyze the sequence of frames. After applying the orange fish mixture (visible in multiple frames, such as @\textcolor{red}{frame 2 at 0.81 seconds where the mixture is at <points x1="355" y1="189" alt="orange fish mixture">orange fish mixture</points>}@ and @\textcolor{red}{frame 5 at 3.25 seconds where it's at <points x2="339" y2="169" alt="orange fish mixture">orange fish mixture</points>)}@, the next action involves wrapping the sushi roll. In @\textcolor{red}{frame 24 at 18.71 seconds}@, the person places clear plastic wrap over the sushi roll (at coordinates around @\textcolor{red}{<points x3="342" y3="169" alt="plastic wrap">plastic wrap</points>}@), and subsequent frames from @\textcolor{red}{frame 25 at 19.52 seconds to frame 27}@ show the person continuing to wrap the sushi roll in plastic. This is followed by using a bamboo rolling mat from @\textcolor{red}{frame 28 at 21.15 seconds to frame 30 at 23.59 seconds}@ to shape the sushi. Therefore, after topping with the orange fish mixture, the person wraps the sushi roll in plastic wrap. The answer is <|begin_of_box|>Wraps the sushi roll in plastic wrap<|end_of_box|>.</answer>
\end{lstlisting}
\end{AIBox}
\caption{Video example 1 (continued)} 
\label{fig:video_example_1_part2}
\end{figure*}

\begin{figure*}[t]
\centering
\begin{AIBox}{Video Example 2}
\begin{minipage}{\textwidth}
    \centering
    \includegraphics[width=\textwidth]{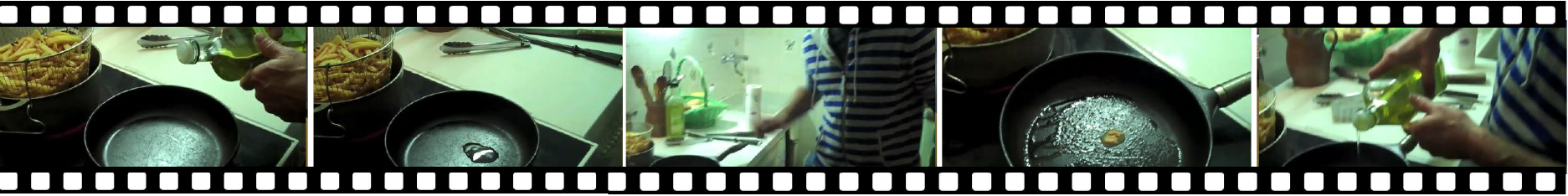}
\end{minipage}

\vspace{0.5em}

\small
\textbf{Question:}

\begin{lstlisting}[breaklines=true, basicstyle=\small\ttfamily, breakindent=0pt, escapechar=@]
How does the camera move when showing the frying pan on the stove? A. It tilts upwards. B. It pans to the left. C. It zooms in closer to the frying pan. D. It moves away from the frying pan.
\end{lstlisting}

\vspace{0.3em}
\textbf{Response:}

\begin{lstlisting}[breaklines=true, basicstyle=\small\ttfamily, breakindent=0pt, escapechar=@]
<think>Got it, let's analyze the camera movement when showing the frying pan. First, I need to check how the camera's perspective changes. Let's look at the frames with the frying pan. 

Looking at the sequence, when the camera focuses on the frying pan, does it zoom in? Let's check the frames. For example, when the oil is poured, the camera might zoom in closer. Wait, let's check the frames where the frying pan is the main focus. The camera seems to move closer, increasing the size of the frying pan in the frame. 

Let's check the coordinates. In earlier frames, the frying pan might be smaller, then as the camera zooms in, it becomes larger. For example, comparing @\textcolor{red}{frame 1 at 0.00 seconds and frame 4 at 2.53 seconds}@, the frying pan's size relative to the frame changes. If it's zooming in, the frying pan would appear larger, meaning the camera is moving closer. So, the camera zooms in closer to the frying pan. Let's confirm. The other options: tilting upwards would change the angle vertically, panning left would move the camera horizontally, moving away would make the pan smaller. The evidence shows the pan becomes larger in the frame, indicating zooming in. 
\end{lstlisting}
\end{AIBox}
\caption{Video example 2 (part 1)} 
\label{fig:video_example_2_part1}
\end{figure*}

\begin{figure*}[t]
\centering
\begin{AIBox}{Video Example 2 (cont.)}
\small
\textbf{Response (continued):}

\begin{lstlisting}[breaklines=true, basicstyle=\small\ttfamily, breakindent=0pt, escapechar=@]
Therefore, the camera moves by zooming in closer to the frying pan.</think><answer>To determine how the camera moves when showing the frying pan on the stove, we analyze the spatiotemporal changes. When focusing on the frying pan, the camera's perspective shifts such that the frying pan appears larger in successive frames. For example, comparing @\textcolor{red}{frame 1 at 0.00 seconds and frame 4 at 2.53 seconds}@, the frying pan at @\textcolor{red}{<points x1="356" y1="273" alt="frying pan">frying pan</points> in frame 1}@ appears smaller relative to the frame than the frying pan at @\textcolor{red}{<points x2="356" y2="262" alt="frying pan">frying pan</points> in frame 4}@. This indicates the camera is moving closer to the frying pan. Additionally, as the sequence progresses, the frying pan's size increases in the frame, confirming that the camera zooms in closer rather than tilting, panning, or moving away. Therefore, the camera moves by zooming in closer to the frying pan. <|begin_of_box|>C<|end_of_box|>.</answer>
\end{lstlisting}
\end{AIBox}
\caption{Video example 2 (part 2)} 
\label{fig:video_example_2_part2}
\end{figure*}

\end{document}